\documentclass{article}

\usepackage{microtype}
\usepackage{graphicx}
\usepackage{subfigure}
\usepackage{booktabs} 

\usepackage{hyperref}

\usepackage[accepted]{preprint}

\usepackage{amsmath}
\usepackage{amssymb}
\usepackage{mathtools}
\usepackage{amsthm}

\usepackage[capitalize,noabbrev]{cleveref}

\theoremstyle{plain}
\newtheorem{theorem}{Theorem}[section]

\newtheorem{lemma}[theorem]{Lemma}

\theoremstyle{definition}

\theoremstyle{remark}

\usepackage[textsize=tiny]{todonotes}

\usepackage{graphicx}
\usepackage{float} 
\usepackage{bm}
\usepackage{enumerate}
\usepackage{enumitem}
\usepackage{multirow}
\usepackage{tikz}
\usetikzlibrary{arrows.meta,positioning} 
\usepackage{adjustbox}
\usepackage{threeparttable}
\usepackage[T1]{fontenc}

\hypersetup{
	colorlinks=true,%
	citecolor={blue!30!black},
	linkcolor={red!50!red},
	urlcolor={green!50!black}
}

\newcommand*\circled[1]{\tikz[baseline=(char.base)]{
    \node[shape=circle,draw,inner sep=0.1pt] (char) {$#1$};}}

\icmltitlerunning{Model-Based Decentralized Policy Optimization}

\begin{document}

\twocolumn[
\icmltitle{Model-Based Decentralized Policy Optimization}



\icmlsetsymbol{equal}{*}

\begin{icmlauthorlist}
\icmlauthor{Hao Luo}{pku}
\icmlauthor{Jiechuan Jiang}{pku}
\icmlauthor{Zongqing Lu}{pku}
\end{icmlauthorlist}

\icmlaffiliation{pku}{School of Computer Science, Peking University}

\icmlcorrespondingauthor{Hao Luo}{lh2000@pku.edu.cn}
\icmlcorrespondingauthor{Zongqing Lu}{zongqing.lu@pku.edu.cn}

\icmlkeywords{Multi-Agent Reinforcement Learning}

\vskip 0.3in
]



\printAffiliationsAndNotice{}  

\begin{abstract}
  
Decentralized policy optimization has been commonly used in cooperative multi-agent tasks. However, since all agents are updating their policies simultaneously, from the perspective of individual agents, the environment is non-stationary, resulting in it being hard to guarantee monotonic policy improvement. To help the policy improvement be stable and monotonic, we propose model-based decentralized policy optimization (MDPO), which incorporates a latent variable function to help construct the transition and reward function from an individual perspective. We theoretically analyze that the policy optimization of MDPO is more stable than model-free decentralized policy optimization. Moreover, due to non-stationarity, the latent variable function is varying and hard to be modeled. We further propose a latent variable prediction method to reduce the error of the latent variable function, which theoretically contributes to the monotonic policy improvement. Empirically, MDPO can indeed obtain superior performance than model-free decentralized policy optimization in a variety of cooperative multi-agent tasks.

\end{abstract}

\section{Introduction}

Decentralized multi-agent reinforcement learning (MARL) has been commonly used in practice for cooperative multi-agent tasks, \textit{e.g.}, traffic signal control \citep{wei2018intellilight}, unmanned aerial vehicles \citep{qie2019joint}, and IoT \citep{cao2020multiagent}, where global information is inaccessible. Independently performing policy optimization using local information, \textit{e.g.}, independent PPO \citep{schulman2017proximal} (IPPO), is one of the most straightforward methods for decentralized MARL. Recent empirical studies \citep{de2020independent,yu2021surprising,benchmark} demonstrate that IPPO performs surprisingly well in several cooperative multi-agent benchmarks, which shows great promise for fully decentralized policy optimization. 

However, since all agents are updating policies, from the perspective of an individual agent, the environment is non-stationary \citep{Zhang2019multi}. Thus, the monotonic policy improvement, which can be achieved by policy optimization in single-agent settings \citep{schulman2015trust,schulman2017proximal}, may not be guaranteed in decentralized MARL. Concretely, in policy optimization, the state visitation frequency is assumed to be stationary since the agent policy is limited to slight updates, which is necessary to guarantee monotonic policy improvement \citep{schulman2015trust}. However, in decentralized multi-agent settings, as all agents are updating policies simultaneously, the state visitation frequency will change largely, which contradicts the fundamental assumption of policy optimization, thus the monotonic improvement of policy optimization may not be preserved.

To address this problem, we resort to exploiting the environment model to stabilize the state visitation frequency and help monotonic policy improvement. However, learning an environment model in decentralized settings is non-trivial, since the information of other agents, \textit{e.g.}, other agents' policies, is not observable and changing. Therefore, we introduce a latent variable to help distinguish different transitions resulting from the unobservable information. And then we build an environment model for each agent, which contains a transition function, a reward function, and a latent variable function that learns the latent variable given observation. The agents are trained using independent policy optimization methods, \textit{e.g.}, TRPO \citep{schulman2015trust} or PPO \citep{de2020independent}, on both the experiences generated by the environment model and collected in the environment.

Since the environment is non-stationary, the latent variable function is also varying during learning. We theoretically show that independently performing policy optimization on experiences generated by the environment model with the varying latent variable function can obtain more stationary observation visitation frequency than on the experiences collected in the non-stationary environment. Thus, independent policy optimization goes more stable on the environment model. 

Moreover, to obtain monotonic improvement, the gap between the return of interacting with the environment and the return predicted by the environment model should be small. We theoretically analyze that the return gap is bounded by the prediction error of the latent variable function. As the latent variable function is varying due to non-stationarity, to minimize the prediction error, we propose a latent variable prediction method that uses the historical variables to predict the future variable. Thus, the latent variable prediction can reduce the return gap and help the monotonic policy improvement.

The proposed algorithm, \textbf{\textit{model-based decentralized policy optimization}} (\textbf{MDPO}), is theoretically grounded and empirically effective for fully decentralized learning. We evaluate MDPO on a variety of cooperative multi-agent tasks, \textit{i.e.}, a stochastic game, multi-agent particle environment (MPE) \citep{lowe2017multi}, multi-agent MuJoCo \citep{mamujoco}, and Google Research Football (GRF)\citep{kurach2020google}. MDPO outperforms the model-free independent policy optimization baseline, and the proposed latent variable prediction additionally obtains performance gain, verifying that MDPO can help stable and monotonic policy improvement in fully decentralized learning.


\section{Preliminaries}

\textbf{Dec-POMDP.} A cooperative multi-agent task is generally modeled as a decentralized partially observable Markov decision process (Dec-POMDP) \citep{oliehoek2016concise}. Specifically, a Dec-POMDP is defined as a tuple $G = \{\mathcal{S},\mathcal{I},{\mathcal{A}},\mathcal{O},\Omega,P,R,\gamma\}$. $\mathcal{S}$ is the state space, $\mathcal{I}$ is the set of agents, and ${\mathcal{A}}={A}_1\times\cdots\times {A}_{|\mathcal{I}|}$ is the joint action space, where ${A}_i$ is the action space for each agent $i$. At each state $s$, each agent $i\in\mathcal{I}$ merely gets access to the observation $o_{i}\in\mathcal{O}$, which is drawn from observation function $\Omega(s,i)$, and selects an action $a_i \in {A}_i$, and all the actions form a joint action $\boldsymbol{a} \in {\mathcal{A}}$. The state transitions to next $s'$ according to the transition function $P(s'|s, \bm{a}):\mathcal{S}\times{\mathcal{A}}\times\mathcal{S}\rightarrow[0,1]$, and all agents receive a shared reward $r=R(s,\bm{a}):\mathcal{S}\times{\mathcal{A}}\rightarrow\mathbb{R}$. The objective is to maximize the expected return $\eta(\boldsymbol{\pi})=\mathbb{E}[\sum_{t=0}^\infty \gamma^t r_t|\rho_0,\boldsymbol{\pi}]$ under the joint policy of all agents $\boldsymbol{\pi}$ and distribution of initial state $\rho_0$, where $\gamma \in [0,1)$ is the discounted factor. The joint policy $\bm{\pi}$ can be represented as the product of each agent's policy $\pi_i$. Also we denote $\bm{\pi}_{-i}$ as the joint policy of all agents except $i$.

\textbf{Fully decentralized learning.} We consider the fully decentralized way to solve the Dec-POMDP \citep{tan1993multi,de2020independent}, where each agent independently learns a policy and executes actions without communication or parameter sharing in both training and execution phases. Since all agents are updating policies, from the perspective of individual agents, the environment is non-stationary, which fundamentally challenges decentralized learning \citep{Zhang2019multi}. The existing decentralized MARL methods are limited. Independent Q-learning (IQL) \citep{tan1993multi} and independent policy optimization, \textit{e.g.}, IPPO \citep{de2020independent}, are the most straightforward fully decentralized algorithms. Despite good empirical performance \citep{benchmark}, due to non-stationarity, these methods lack theoretical guarantees. IQL has no convergence guarantee, to the best of our knowledge. Although there has been some study \citep{sun2022monotonic}, IPPO may not guarantee policy improvement by independent policy optimization, since the assumption of stationary state visitation frequency for policy optimization may not hold in fully decentralized settings, which we will discuss in the following.

\textbf{Monotonic policy improvement.} In Dec-POMDP, from a centralized perspective, we can obtain a TRPO objective \citep{schulman2015trust} of the joint policy $\bm{\pi}$ for the monotonic improvement,
\begin{small}
\begin{align*}
    &\eta({\bm{\pi}^{\operatorname{new}}} )-\eta({\bm{\pi}^{\operatorname{old}} } ) \\
    \ge & \sum_{s,\bm{a}} {\color{red!30!orange}{\rho}^{\bm{\pi}^{\operatorname{new}}}}(s)  \bm{\pi}^{\operatorname{new}}(\bm{a} | s) A^{\bm{\pi}^{\operatorname{old}}}(s,\bm{a}) - C \cdot D_{\operatorname{KL}}^{\operatorname{max}}(\bm{\pi}^{\operatorname{old}} \| \bm{\pi}^{\operatorname{new}}) \\
    \gtrapprox & \sum_{s,\bm{a}} {\color{blue!70!black} \rho^{\bm{\pi}^{\operatorname{old}}}}(s)  \bm{\pi}^{\operatorname{new}}(\bm{a} | s) A^{\bm{\pi}^{\operatorname{old}}}(s,\bm{a}) - C \cdot D_{\operatorname{KL}}^{\operatorname{max}}(\bm{\pi}^{\operatorname{old}} \| \bm{\pi}^{\operatorname{new}}), 
\end{align*}    
\end{small}

where ${\rho}^{\bm{\pi}^{\operatorname{old}}}(s) = \sum_{t = 0} \gamma^t \operatorname{Pr}(s_t = s| {\bm{\pi}^{\operatorname{old}} })$ is the discounted state visitation frequency given $\bm{\pi}^{\operatorname{old}}$, similarly for ${\rho}^{\bm{\pi}^{\operatorname{new}}}(s)$, $A^{\bm{\pi}^{\operatorname{old}}}$ is the advantage function under $\bm{\pi}^{\operatorname{old}}$, $D_{\operatorname{KL}}^{\operatorname{max}}(\bm{\pi}^{\operatorname{old}} \| \bm{\pi}^{\operatorname{new}}) = \max_{s} D_{\operatorname{KL}}(\bm{\pi}^{\operatorname{old}}(\cdot|s) \| \bm{\pi}^{\operatorname{new}}(\cdot|s) )$, and $C$ is a constant. From ${\color{red!30!orange}{\rho}^{\bm{\pi}^{\operatorname{new}}}}$ to ${\color{blue!70!black} \rho^{\bm{\pi}^{\operatorname{old}}}}$ is an approximation or assumption \citep{schulman2015trust}. As ${\rho}^{\bm{\pi}^{\operatorname{new}}}$ is \textit{unknown} and the policy is limited to slight updates, ${\rho}^{\bm{\pi}^{\operatorname{new}}}$ is approximated by ${\rho}^{\bm{\pi}^{\operatorname{old}}}$. However, in fully decentralized MARL, this assumption may not hold, as all agents are updating their policies simultaneously and their joint policy may change significantly especially when the number of agents is large. This will severely affect the performance of independent policy optimization. Although we can constrain the policy update of each agent to be slight like TRPO, this leads to much slower convergence, especially in fully decentralized MARL, where the joint policy has a much larger search space and is merely optimized by independent learning of individual agents.

\section{Methodology} 

In this paper, we provide a novel perspective and resort to the environment model to bridge the gap between ${\rho^{\bm{\pi}^{\operatorname{new}}}}$ and  ${\rho^{\bm{\pi}^{\operatorname{old}}}}$ for each agent such that the monotonic joint policy improvement can be potentially achieved by fully decentralized policy optimization. 

As illustrated in the following, we turn the learning process into a Dyna-style \citep{sutton1990integrated} decentralized model-based method with the {\color{green!65!blue}green} path\footnote{Related work on model-based MARL can be found in Appendix~\ref{app:related}. However, none of the existing work considers exploiting the environment model to help fully decentralized policy optimization.}. Each agent $i$ additionally learns a decentralized model using local information from policy rollout and can optionally perform policy optimization on the experiences from model rollout. When optimizing policy with model rollout, we essentially have $\circled{1}$, which means the state visitation frequency in model rollout ($\rho_{\operatorname{model}}$) is more stable. Thus, the approximation from ${\color{red!30!orange}{\rho}^{\bm{\pi}^{\operatorname{new}}}}$ to ${\color{blue!70!black} \rho^{\bm{\pi}^{\operatorname{old}}}}$ becomes acceptable under looser constraints of policy update. Further, we can bound the gap between the returns of policy rollout ($\eta$) and model rollout ($\eta^{\operatorname{model}}$), formally in $\circled{2}$. Once the bound ($\mathcal{B}$) is controllable throughout the learning process, it can potentially guarantee the monotonic improvement of the joint policy in the real environment.

\begin{figure}[h]
\centering
\adjustbox{width=0.48\textwidth}{
\small
\begin{tikzpicture}[node distance=.7cm, auto]
 \node[] (pom) {\color{green!65!blue} policy optimization $\pi_i^{\operatorname{model}}$};
 \node[left = of pom] (mre) {\color{green!65!blue} model rollout}
    edge[-stealth,thick, green!65!blue] (pom.west);
 \node[left = of mre] (model) {\color{green!65!blue} model}
    edge[-stealth,thick, green!65!blue] node[below] {$\pi_{i}^{\operatorname{old}}$} (mre.west);
 
 \node[above = of pom] (poe) {policy optimization $\boldsymbol{\pi}^{\operatorname{new}}$};
 \node[left =of poe] (pre) {policy rollout}
    edge [-stealth,thick] (poe.west);
 \node[left=of pre]  (env) {environment}
    edge[-stealth,thick] node[above] {$\boldsymbol{\pi}^{\operatorname{old}}$} (pre.west);

 \path[->] (pre) edge[-stealth,thick, green!65!blue] (model);
 \draw [-{Latex[round,open]},double,thick] (-2.2,-0.3) -- (-2.2,-1.2) ;
 
 \draw (-5.2, -1.6) node[anchor=west] {$\circled{1}$ \(\left\|{\color{red!30!orange}\rho^{\boldsymbol{\pi}^{\operatorname{new}}}} - {\color{blue!70!black}\rho^{\boldsymbol{\pi}^{\operatorname{old}}}}\right\|>\left\|\rho_{\operatorname{model}}^{{\color{green!65!blue}\pi_{i}^{\operatorname{model}}}} -\color{black}{\rho_{\operatorname{model}}^{\pi_i^{\operatorname{old}}}} \right\|\)};
 
 \draw (-5.2, -2.5) node[anchor=west] (tag) {$\circled{2}$ \(\left|{\eta(\pi_{i}^{\operatorname{model}},\bm{\pi}_{-i}^{\operatorname{new}}) - \eta^{\operatorname{model}}(\pi_{i}^{\operatorname{model}})}\right| < \mathcal{B}\)};

\end{tikzpicture}
}
\end{figure}
Thus, $\circled{1}$ and $\circled{2}$ together highlight the potential benefits of incorporating an environment model for decentralized policy optimization. In the following, we discuss how to learn such a decentralized model, theoretically investigate its benefits for decentralized policy optimization, and analyze the return bound for monotonic policy improvement.

\subsection{Latent Variable Model} 

In decentralized learning, from the perspective of each agent $i$, the transition function and reward function are respectively,
\begin{align*}
P_i(s'|s,a_i) = \mathbb{E}_{\boldsymbol{a}_{-i}\sim {\pi}_{-i}}P(s'|s,a,\boldsymbol{a}_{-i}) 
\end{align*}
and
\begin{align*}
R_i(s,a_i) = \mathbb{E}_{\boldsymbol{a}_{-i}\sim {\pi}_{-i}}R(s,a,\boldsymbol{a}_{-i}),
\end{align*}
where $\boldsymbol{a}_{-i}$ denotes the joint action of all agents except $i$. As other agents are also updating their policies, $P_i$ and $R_i$ are varying throughout the learning process, which is the well-known non-stationarity problem. Moreover, as each agent $i$ usually obtains observation instead of state in decentralized learning, the model can only be learned on $(o_i,a_i,o'_i,r)$. Thus, it is challenging to construct an environment model from the perspective of an individual agent. 

To build a decentralized environment model, we introduce a latent variable $z_i$, which helps distinguish different transitions resulting from varying unobservable information of the full state and other agents' policies. Then the transition function and the reward function can be redefined as:
\begin{equation*}
P_i(o_i'|o_i,a_i,z_i) \quad \text{and} \quad R_i(o_i,a_i,z_i).
\end{equation*}
As we discuss fully decentralized learning, we drop the subscript of $i$ for simplicity in the following.

To model the transition function and the reward function with the latent variable, we define the latent variable function from the perspective of an individual agent, $\psi(z|o)$, which indicates the probability of latent variable $z$ given observation $o$. 
As $z$ is related to the policies of other agents, $\psi(z|o)$ also varies during policy updates. 
A latent variable model consists of three modules: transition function $P_\theta$, reward function $R_\phi$, and latent variable function $\psi_\omega$, to predict the next observation and reward. As the impact of unobservable information is designed to merely reflect on the latent variable, although other agents update their policies, the transition function and reward function stay constant and only the latent variable function varies. We learn such a model by maximizing the likelihood of experiences of policy rollout $\mathcal{D}$, and the objective is
\begin{align}\label{eq:objective}
    &\underset{\theta,\omega,\phi}{\min}\left(\mathcal{L}_{\operatorname{rew}} + c_{o^{\prime}} \cdot {\mathcal{L}}_{\operatorname{trans}}\right),\\
    &\mathcal{L}_{\operatorname{rew}} = \mathbb{E}_{(o,a,o^\prime,r)\sim\mathcal{D},z\sim\psi_\omega(\cdot|o)}(R_\phi\left(o,a,o^\prime,z\right) - r)^2, \notag \\
    &\mathcal{L}_{\operatorname{trans}}  = \mathbb{E}_{(o,a,o^\prime,r)\sim\mathcal{D},z\sim\psi_\omega(\cdot|o)} - P_\theta\left(o^\prime|o,a,z\right), \notag 
\end{align}
where coefficient $c_{o\prime}$ is used to balance the scale of $\mathcal{L}_{\operatorname{trans}}$ and $\mathcal{L}_{\operatorname{rew}}$.
We examine the correlation between the latent variable learned end-to-end and inaccessible information in a simple setting, and the learned latent variable is indeed correlated with the inaccessible information. More details can be found in Appedix~\ref{sec:verify}.

Moreover, when using the learned latent variable model to train an agent, we adopt $k$-step branched model rollout in MBPO \citep{janner2019trust} to avoid compounding model error due to long-horizon rollout. Concretely, at each policy update of an agent, we sample $h$-step length experiences $\{( o_1,a_1,o^\prime_{1},r_1),\cdots,( o_h,a_h,o^\prime_{h},r_h)\}$ from policy rollout $\mathcal{D}$ and perform $k$-step model rollout starting from the last observation $o^\prime_{ h}$ under current policy $\pi$. The policy $\pi$ is updated on the merged $(h+k)$-step experiences $\{( o_1,a_1,o^\prime_{1},r_1),\cdots,( o_{h+k},a_{h+k},o^{\prime}_{h+k},r_{h+k})\}$ by policy optimization, \textit{e.g.}, PPO \citep{schulman2015trust}.

\subsection{Stable Policy Optimization on Model}\label{sec.3.2}
Now, we turn to analyze the benefits of such a model-based method over model-free independent policy optimization. We first theoretically analyze that independently performing policy optimization \textit{e.g.,}, TRPO \citep{schulman2015trust} or PPO \citep{schulman2017proximal}, on the latent variable model can make the learning process more stable.

In decentralized learning, from the perspective of an agent, given the \textit{true} latent variable function $\psi$, the discounted observation visitation frequency of $\mathcal{D}$ obtained by policy rollout is defined as 
\begin{equation*}
\rho ^{\pi ,\psi}\left( o \right) =\rho _{0}^{\pi ,\psi}\left( o \right) +\gamma \rho _{1}^{\pi ,\psi}\left( o \right) +\gamma ^2\rho _{2}^{\pi ,\psi}\left( o \right) +\cdots ,
\end{equation*}
where $\rho _{t}^{\pi ,\psi}\left(o\right) \triangleq Pr(o_t=o)$ and $o_t$ is the observation at timestep $t$ of experience from $\mathcal{D}$. Note that $\rho ^{\pi ,\psi}\left( o \right)$ is an unbiased estimate of discounted observation visitation frequency when interacting in the environment. Similarly, $\rho^{\pi ,\psi_{\omega}}$ denotes the discounted observation visitation frequency for experiences obtained by model rollout. During the learning process, $\pi^n$ and $\psi^n$ respectively denote the policy and latent variable function after the $n$th policy update. Then, we have the following theorem. All proofs are available in Appendix~\ref{sec:proofs}.
\begin{theorem}\label{th.1}
Define $\Delta \rho^n(o)\triangleq\rho^{\pi^{n},\psi^{n}}(o)-\rho^{\pi^{n-1},\psi^{n-1}}(o)$. Denote $ \left\| \Delta \rho ^n \right\| \triangleq \underset{o}{\max}|\rho^n(o)-\psi^{n-1}(o)|$, similarly for $\left\| \Delta \pi ^n \right\|$ and $\left\| \Delta \psi ^n \right\|$. It holds that,
\begin{equation*}
    \left\| \Delta \rho ^n \right\| \le C\left( \mathcal{E}_\pi + \mathcal{E}_\psi \right),
\end{equation*}
where $\mathcal{E}_\pi \triangleq \max\limits_n\left\|\Delta\pi^n\right\|$, $\mathcal{E}_\psi \triangleq \max\limits_n\left\|\Delta\psi^n\right\|$ and $C$ is a constant. Assume $\psi^n_\omega = (1-\alpha)\psi^{n-1}_\omega + \alpha\psi^n$ and $\psi^0_\omega = \psi^0$ \footnote{Since $\psi$ is varying and $\psi_\omega$ is continuously updated using the experiences from several recent policy rollouts, we use the form of soft-update for the relation between $\psi$ and $\psi_{\omega}$.}. It holds that $\mathcal{E}_\psi > \mathcal{E}_{\psi_\omega}$ and the bound above is lower when substituting $\psi$ with $\psi_\omega$.
\end{theorem}

According to Theorem \ref{th.1}, the divergence of discounted observation visitation frequency is bounded by the divergence of policy and latent variable function. Again, the policy divergence can be constrained via policy optimization, like TRPO. Thus, the main difference lies in the divergence of latent variable function. As indicated by Theorem \ref{th.1}, the learned latent variable function $\psi_\omega$ has a smaller divergence between consecutive policy rollouts than the true latent variable function $\psi$. \textit{Therefore, independent policy optimization on experiences generated by the latent variable model can obtain more stationary observation visitation frequency than on the experiences collected in the varying environment, so the learning process of independent policy optimization becomes stable on the model.}

\begin{figure*}[ht]
        \centering
        \includegraphics[width=0.75\textwidth]{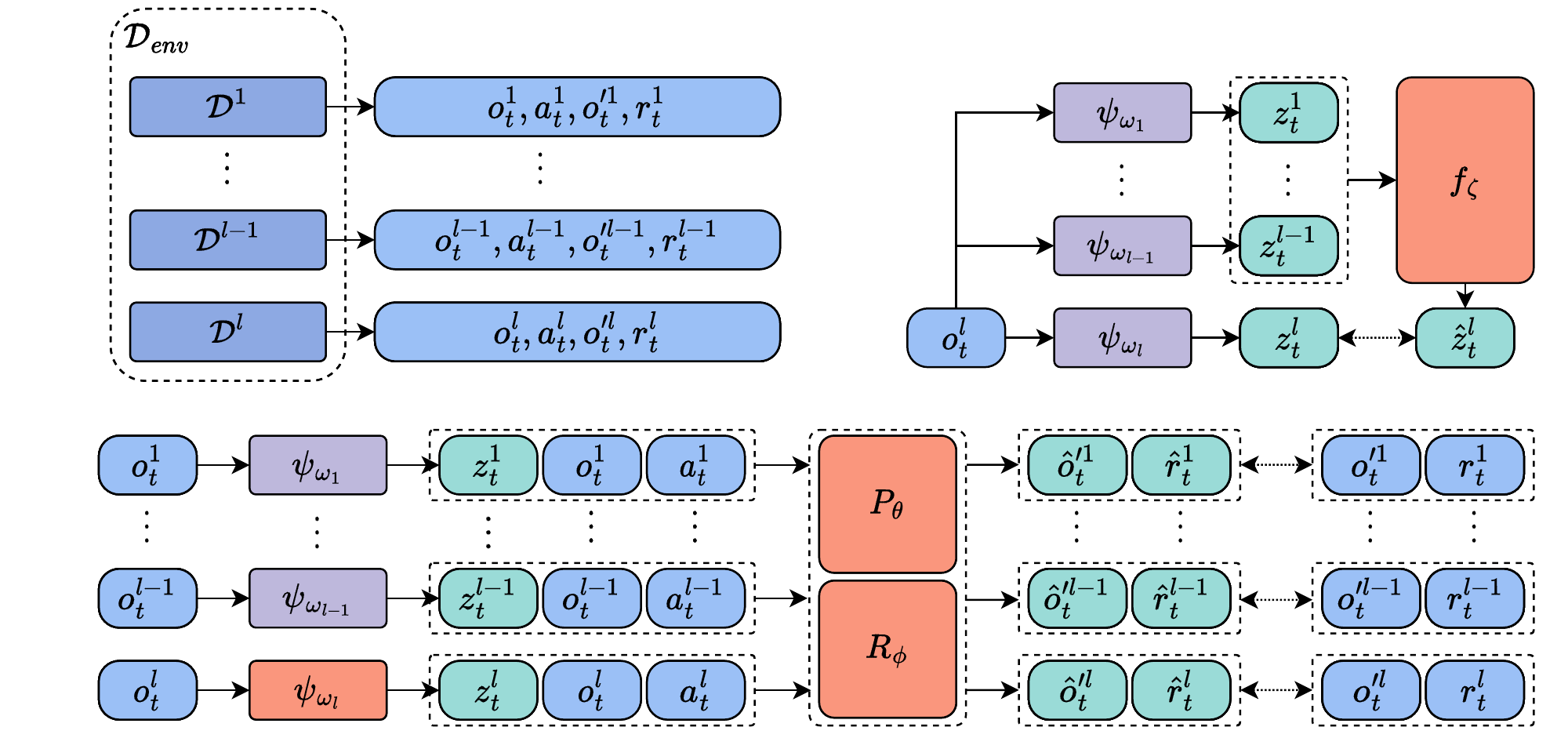}
        \vspace{-0.2cm}
        \caption{The environment model includes four modules: transition function $P_\theta$, reward function $R_\phi$, latent variable prediction function $f_\zeta$, and latent variable functions $\{\psi_{\omega_1},\cdots,\psi_{\omega_l}\}$ over $l$ consecutive policy rollouts. For learning, each agent maintains the experiences of $l$ consecutive policy rollouts, $P_\theta$ and $R_\phi$ learn on $l$ consecutive policy rollouts, $\psi_{\omega_l}$ learns on the experiences of $l$th policy rollout, and $f_\zeta$ learns to predict $l$th latent variable given $l-1$ latent variable functions.}
        \label{fig:model with prediction}
\end{figure*}

\subsection{Return Bounds}
We then analyze the bound of return gap between interacting in the environment and interacting with the model. If the return improvement of interacting with the model is higher than the bound, the agent can obtain the monotonic policy improvement when interacting in the environment. 

However, the return of interacting in the environment is hard to analyze in decentralized learning since the policies of other agents are inaccessible, we turn to analyze the return in policy rollout, which is an unbiased estimate of expected return in the environment.

Several bounds have been introduced in MBPO \citep{janner2019trust} for the return bound analysis, which however are not sufficient in decentralized learning. Thus, we need to introduce two new bounds that indicate the divergence of the latent variable function between consecutive policy rollouts and the error of the learned latent variable function. Here, we analyze the return bound with reward bound ($r_{\max}$), transition error ($\epsilon _\theta$), policy divergence ($\epsilon _\pi$), \textit{latent variable function divergence} ($\textcolor{blue!70!black}{\epsilon_\psi}$), and \textit{learned latent variable function error} ($\textcolor{red!30!orange}{\epsilon_\omega}$):
\begin{align*}
    r_{\max} &\triangleq \underset{o,a,z}{\max}\max\{R(o,a,z),R_\phi(o,a,z)\}, \\
    \epsilon _\theta &\triangleq \underset{t}{\max}D_{TV}\left(P^t\|P^t_\theta\right)\footnotemark[3] , \\ 
    \epsilon _\pi &\triangleq \underset{o}{\max}D_{TV}\left( \pi \|\pi^n \right), \\
    \textcolor{blue!70!black}{\epsilon_\psi} &\triangleq \underset{o}{\max}D_{TV}\left( \psi^n \|\psi^{n+1} \right),\\
    \textcolor{red!30!orange}{\epsilon_\omega} &\triangleq \underset{o}{\max}D_{TV}\left( \psi^n \|\psi^n_\omega \right),
\end{align*}
\footnotetext[3]{$D_{TV}$ means total variance and $D_{TV}\left(P^t\|P^t_\theta\right)$ means $\mathbb{E}_{a_t,z_t\sim(\pi,\psi^n_\omega)} D_{TV}\left( P\left( o_{t+1}|o_t,a_t\right) \|P_\theta\left( o_{t+1}|o_t,a_t,z_t \right) \right)$} 

\vspace{-4mm}
where $\psi^n$, $\psi^n_\omega$, and $\pi^n$ respectively refer to the true latent variable function, the learned latent variable function, and the policy of the $n$th policy rollout.

Additionally, we use several notations to represent different returns. The return in $n$th policy rollout with the true latent variable function $\psi$ is denoted as $\eta(\pi,\psi)$, the return in model rollout with the $n$th learned latent variable function $\psi_\omega$ is denoted as $\eta^{model}(\pi,\psi_\omega^n)$, and the return in $k$-step branched model rollout with $h$-step experiences of $n$th policy rollout is denoted as $\eta^{branch}((\pi^n,\pi),(\psi,\psi_\omega^n))$. Now we analyze the return bound of model rollout and branched model rollout with the newly introduced $\epsilon_\psi$ and $\epsilon_\omega$ in the following two theorems.

\begin{theorem}\label{th.2}
Denote the return gap between $n+1$th policy rollout and model rollout with $n$th learned model as $\left| \eta \left( \pi ,\psi ^{n+1} \right) -\eta ^{\operatorname{model}}\left( \pi ,\psi _{\omega}^{n} \right) \right| \triangleq \Delta \eta $ ,which is bounded as:
\begin{align*}
 \Delta \eta \le \underset{C\left(\epsilon _\theta,\epsilon _{\pi}, \textcolor{red!30!orange}{\epsilon _\omega},\textcolor{blue!70!black}{\epsilon _{\psi}} \right)}{\underbrace{\frac{2r_{\max}}{\left( 1-\gamma \right) ^2}\left( \gamma \epsilon_\theta+2\epsilon_{\pi}+2\epsilon _\omega+ \epsilon _{\psi} \right) }}.
\end{align*}
\end{theorem}
\begin{theorem}\label{th.3}
 Denote the return gap between $n+1$th policy rollout and branched model rollout with $n$th learned model as $\left| \eta \left( \pi ,\psi ^{n+1} \right) -\eta ^{\operatorname{branch}}\left( \left( \pi^n,\pi \right) ,\left( \psi ^n,\psi _{\omega}^{n} \right) \right) \right| \triangleq \Delta \eta^{\operatorname{branch}} $, which is bounded as:
 
\begin{equation*}
\Delta \eta^{\operatorname{branch}} \le C\left(\epsilon _\theta,\epsilon _{\pi}, \textcolor{red!30!orange}{\epsilon _\omega}, \textcolor{blue!70!black}{\epsilon _{\psi}} \right).
\end{equation*}
 
\end{theorem}

According to Theorem \ref{th.2} and \ref{th.3}, we can guarantee the monotonic improvement in the environment via improving the return in model rollout or branched model rollout beyond a bound linear to $(\epsilon_\pi,\epsilon_\theta,\epsilon_\omega,\epsilon_\psi)$. In these bounds, $\epsilon_\theta$ and $\epsilon _\omega$ are limited via supervised learning and $\epsilon _{\pi}$ is constrained by policy optimization. However, $\epsilon_{\psi}$ is left \textit{unrestricted}. In the following, we try to find a better bound in which all elements are controllable.

\subsection{Latent Variable Prediction}\label{sec.3.4}

In order to restrict the impact of divergence of the latent variable function, we introduce one new error bound, which measures the divergence between the learned latent variable function and the true latent variable function in incoming policy rollout. Formally, such an error bound in $n$th policy rollout is defined as:
\begin{equation*}
\textcolor{green!65!blue}{\hat{\epsilon}_\omega}\triangleq \underset{o}{\max}D_{TV}\left( \psi^{n+1} \|\psi^n_\omega \right).
\end{equation*}
Now we use $\hat{\epsilon}_\omega$ in place of $\epsilon_\psi$ to analyze the return bound of model rollout and branched model rollout again in the following two theorems.
\begin{theorem}\label{th.4}
Denote the return gap of $n+1$th policy rollout and model rollout with $n$th learned model as $\left| \eta \left( \pi ,\psi ^{n+1} \right) -\eta ^{\operatorname{model}}\left( \pi ,\psi _{\omega}^{n} \right) \right| \triangleq \Delta \eta$, which is bounded as:
\begin{equation*}
\Delta \eta \le C\left(\epsilon _\theta,\epsilon _{\pi},\textcolor{green!65!blue}{\hat{\epsilon}_\omega} \right).
\end{equation*}
\end{theorem}

\begin{theorem}\label{th.5}
The return gap of $n+1$th policy rollout and branched model rollout with $n$th learned model as $\left| \eta \left( \pi ,\psi ^{n+1} \right) -\eta ^{\operatorname{branch}}\left( \left( \pi^n,\pi \right) ,\left( \psi ^n,\psi _{\omega}^{n} \right) \right) \right| \triangleq \Delta \eta^{\operatorname{branch}}$ is bounded as:
\begin{equation*}
\Delta \eta^{\operatorname{branch}} \le C\left(\epsilon _\theta,\epsilon _{\pi}, \textcolor{red!30!orange}{\epsilon _\omega},\textcolor{green!65!blue}{\hat{\epsilon}_\omega} \right).
\end{equation*}
\end{theorem}

Now all elements of the bounds are controllable, once we can constrain $\hat{\epsilon}_\omega$ in the learning process. To achieve this, we introduce a latent variable prediction function, which predicts the latent variable distribution given observation $o$ in incoming policy rollout via latent variable distributions of $o$ in the latest $l-1$ policy rollouts. However, as the true latent variable function cannot be obtained directly for an agent, the latent variable prediction function $f_\zeta$ can instead minimize:
\begin{equation*}
\underset{o}{\max}D_{TV}\left(\psi_{\omega_l}(o)\|f\left(\psi_{\omega_1}(o),\cdots,\psi_{\omega_{l-1}}(o)\right)\right).
\end{equation*}
With such a latent variable prediction function, $\hat{\epsilon}_\omega$ is controllable.


\subsection{Algorithm}

With all the theoretical analysis and discussions above, we are ready to present the learning of \textit{model-based decentralized policy optimization} (MDPO).

As illustrated in Figure~\ref{fig:model with prediction}, the environment model consists of transition function $P_\theta$, reward function $R_\phi$, latent variable prediction function $f_\zeta$, and latent variable functions $\{\psi_{\omega_1},\cdots,\psi_{\omega_l}\}$ over recent $l$ consecutive policy rollouts. The experiences of the $l$ consecutive policy rollouts $\mathcal{D}_{env} =\{\mathcal{D}^1,\cdots, \mathcal{D}^l\}$ are also stored.

After the latest policy rollout $l$, we update the transition function and reward function, and learn the latent variable function $\psi_{\omega_l}$ of policy rollout $l$ by optimizing the objective: 
\begin{align}\label{e.1}
     &\underset{\theta,\omega_l,\phi}{\min} \sum_{j=1}^l \left(\mathcal{L}^j_{\operatorname{rew}}+c_{o^{\prime}} \cdot \mathcal{L}^j_{\operatorname{trans}} \right), \\
      &\mathcal{L}^j_{\operatorname{rew}} = \mathbb{E}_{(o,a,o^\prime,r)\sim\mathcal{D}^j,z\sim\psi_{\omega_j}(o)} \left(R_\phi\left(o,a,z\right) - r\right)^2, \notag \\
      &\mathcal{L}^j_{\operatorname{trans}} = \mathbb{E}_{(o,a,o^\prime,r)\sim\mathcal{D}^j,z\sim\psi_{\omega_j}(o)} - P_\theta\left(o^\prime|o,a,z\right). \notag 
\end{align}
In \eqref{e.1}, $\psi_{\omega_l}$ is obtained by updating $\psi_{\omega_{l-1}}$ using $\mathcal{D}^l$, while $P_\theta$ and $R_\phi$ are updated using $\mathcal{D}_{env}$ to make sure they are stable across policy rollouts. Then, the latent variable prediction function $f_\zeta$ is updated using $\mathcal{D}^l$ by optimizing the objective:
\begin{equation}\label{e.2}
   \underset{\zeta}{\max} \ \mathbb{E}_{o\sim\mathcal{D}^l,z^{i}\sim \psi_{\omega_i}(o)}\left[f_\zeta\left(z^{l}|z^{1},\cdots,z^{l-1}\right)\right]. 
\end{equation}
For model rollout, the model predicts the transition in incoming policy rollout given observation $o$ and action $a$ via $l-1$ latest learned latent variable functions ($\psi_2,\cdots,\psi_l$) as:
\begin{align}\label{e.3}
  & z\sim f_\zeta\left(\cdot | z^{2}\sim\psi_{\omega_2}(o),\cdots,z^{l}\sim\psi_{\omega_{l}}(o)\right), \notag \\
  & \hat{o}^\prime\sim P_\theta\left(o,a,z\right),\\ 
  & \hat{r}=R_\phi\left(o,a,z\right) \notag.  
\end{align}
Finally, the policy is updated using the branched model rollout by policy optimization, such as PPO or TRPO. We summarize the full learning procedure of MDPO in Algorithm \ref{al2}.

\begin{algorithm}[!t]
\caption{\textbf{MDPO}}
\label{al2}
\begin{algorithmic}[1]
\STATE \textbf{Initiate}  $\mathcal{D}_{env}=\{\mathcal{D}^1,\cdots,\mathcal{D}^l\}$, $\pi$, $P_{\theta}$, $R_{\phi}$, $ \Psi=\{\psi_{\omega_1},\cdots,\psi_{\omega_l}\}$, $f_\zeta$.
\REPEAT
\STATE  policy rollout in environment and obtain $\mathcal{D}^l$
{\color{black}\STATE  optimize $P_{\theta}$,$R_{\phi}$ and $\psi_{\omega_l}$ on $\mathcal{D}_{env}$ with \eqref{e.1}}
{\color{black}\STATE  optimize prediction function $f_\zeta$ on $\mathcal{D}^l$ with \eqref{e.2}}
{\color{black}\STATE  obtain branched model rollout $\mathcal{D}_{rollout}$ based on $\mathcal{D}^l$ using $P_{\theta}$, $R_{\phi}$, $\pi$, $\Psi$, and $f_\zeta$ with \eqref{e.3}}
\STATE  optimize policy $\pi$ using $\mathcal{D}_{rollout}$ by PPO or TRPO
\FOR{$j \gets 1 , \ldots, l-1$}
\STATE $\mathcal{D}^j\leftarrow\mathcal{D}^{j+1}$, $\psi_{\omega_j}\leftarrow\psi_{\omega_{j+1}}$
\ENDFOR
\UNTIL{terminate}		
\end{algorithmic}
\end{algorithm}

\section{Experiments}


For evaluation, we compare MDPO, MDPO without latent variable prediction (denoted by MDPO w/o prediction), and independent PPO (IPPO) \citep{schulman2017proximal} on a set of cooperative multi-agent tasks including a stochastic game, multi-agent particle environment (MPE) \citep{lowe2017multi}, and multi-agent MuJoCo \citep{peng2021facmac}, and additionally compare with independent TRPO (ITRPO) \citep{schulman2015trust} in Google Research Football (GRF) \citep{kurach2020google}. We do not consider StarCraft multi-agent challenge (SMAC) \citep{samvelyan2019starcraft}, because IPPO has been shown to perform very well in SMAC \citep{de2020independent,benchmark}, close enough to centralized training with decentralized execution methods like QMIX \citep{rashid2018qmix} and MAPPO \citep{yu2021surprising}. Thus, the gain of MDPO may not be clearly evidenced there.

By experiments, we try to answer the following three questions: 
\vspace{-2mm}
\begin{itemize}
\setlength{\itemsep}{0pt}
    \item[1.] \textit{Does the latent variable model help to generate experiences with more stationary observation visitation frequency experimentally?}
    \item[2.] \textit{Does latent variable prediction help to control $\hat{\epsilon}_\omega$?}
    \item[3.] \textit{Does MDPO help to improve performance in decentralized learning?}
\end{itemize}
\vspace{-2mm}

For a fair comparison, the network architecture and hyperparameters are the same for IPPO and MDPO. The number of environment steps taken in each round (policy rollout, network update) is consistent and thus we compare the performance of methods under the same number of environment steps and policy updates. Note that since we consider fully decentralized learning, for all methods, \textit{agents do not use parameter-sharing.} Indeed, parameter-sharing should not be allowed in decentralized learning \citep{terry2020revisiting}. More details on experiment settings, implementation, and hyperparameters are available in Appendix~\ref{sec:hyperparameters}.  All results are presented using the mean and standard deviation of five runs with different random seeds.

\begin{figure*}[ht]
        \centering
        \includegraphics[width=.8\textwidth]{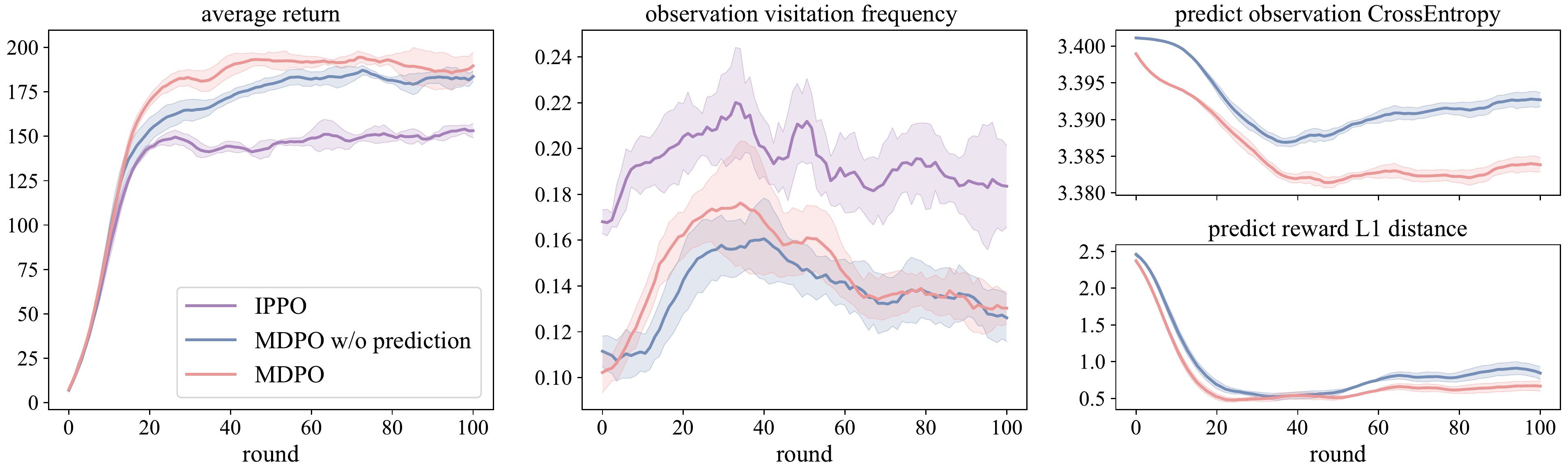}
        \vspace{-0.2cm}
        \caption{Learning curves of MDPO compared with MDPO w/o prediction and IPPO on the stochastic game: average return (left), observation visitation frequency divergence (mid), and model prediction errors (right). Each round is 1600 environment steps.}
        \label{matrix}
\end{figure*}


\subsection{Stochastic Game}

The stochastic game is a cooperative game with 30 observations (states), 3 agents, and 5 actions for each agent, and every episode consists of 40 steps. The transition function and the shared reward function are randomly generated. The game is chosen to verify our theoretical results.  

Figure \ref{matrix} (left) shows the learning curves of MDPO, MDPO w/o prediction, and IPPO, among which MDPO performs better throughout the learning process. With a finite observation space in this game, we calculate the divergence of observation visitation frequencies ($\left\|\Delta\rho\right\|$ in Section \ref{sec.3.2}) in consecutive rollouts. Concretely, we calculate the L1 distance of observation visitation frequency over all observations in consecutive rollouts (policy rollouts for IPPO and branched model rollouts for MDPO), and their curves are shown in Figure \ref{matrix} (mid). We can see that the latent variable model generates experiences with more stationary observation visitation frequency than IPPO, which is consistent with Theorem \ref{th.1}. This may account for the superior performance of MDPO w/o prediction over IPPO.

We also examine how well the latent variable prediction helps to control the prediction error ($\hat{\epsilon}_\omega$). As the real latent variable function is inaccessible, we examine $\hat{\epsilon}_\omega$ by comparing how well the learned environment model predicts with and without latent variable prediction. Specifically, we measure the mean cross-entropy of the next observation distribution predicted by the $n$th learned model and the ground truth in the $n+1$th round, and the mean L1 distance of predicted reward and ground truth reward. The curves are shown in Figure \ref{matrix} (right). The lower prediction error of MDPO indicates that latent variable prediction error ($\hat{\epsilon}_\omega$) is controlled at a lower level than without latent variable prediction. 
Moreover, as shown in Figure \ref{matrix} (mid), the divergence of observation visitation frequency of MDPO and MDPO w/o prediction are similar but much lower than IPPO, which indicates $\epsilon_\omega$ is still under control in MDPO. This indicates that MDPO can well control both $\epsilon_\omega$ and $\hat{\epsilon}_\omega$. 

As MDPO helps to handle non-stationarity in multi-agent settings from the perspective of an individual agent, it will be natural to also apply MDPO to non-stationary single-agent settings. So, we modify this stochastic game into a non-stationary single-agent game and show that MDPO also outperforms the baselines. More details are available in Appendix~\ref{app:more}.

\subsection{MPE}

MPE is a multi-agent environment with continuous observation. In our MPE tasks, agents observe their own positions, velocity, and others' relative positions. And agents are expected to fulfill a certain goal via controlling their accelerations in every direction which is continuous in our experiments. Two tasks of MPE, 4-agent Cooperative Navigation and 5-agent Regular Polygon Control, are chosen for performance comparison.
In 4-agent Cooperative Navigation, 4 agents learn to cooperate to reach 4 landmarks respectively. In 5-agent Regular Polygon Control, 4 agents learn to cooperate with another agent, which is controlled by a fixed policy, aiming to form a regular pentagon, and the reward is given according to the similarity to a regular pentagon.

\begin{figure}[t]
        \centering
        \includegraphics[width=1\linewidth]{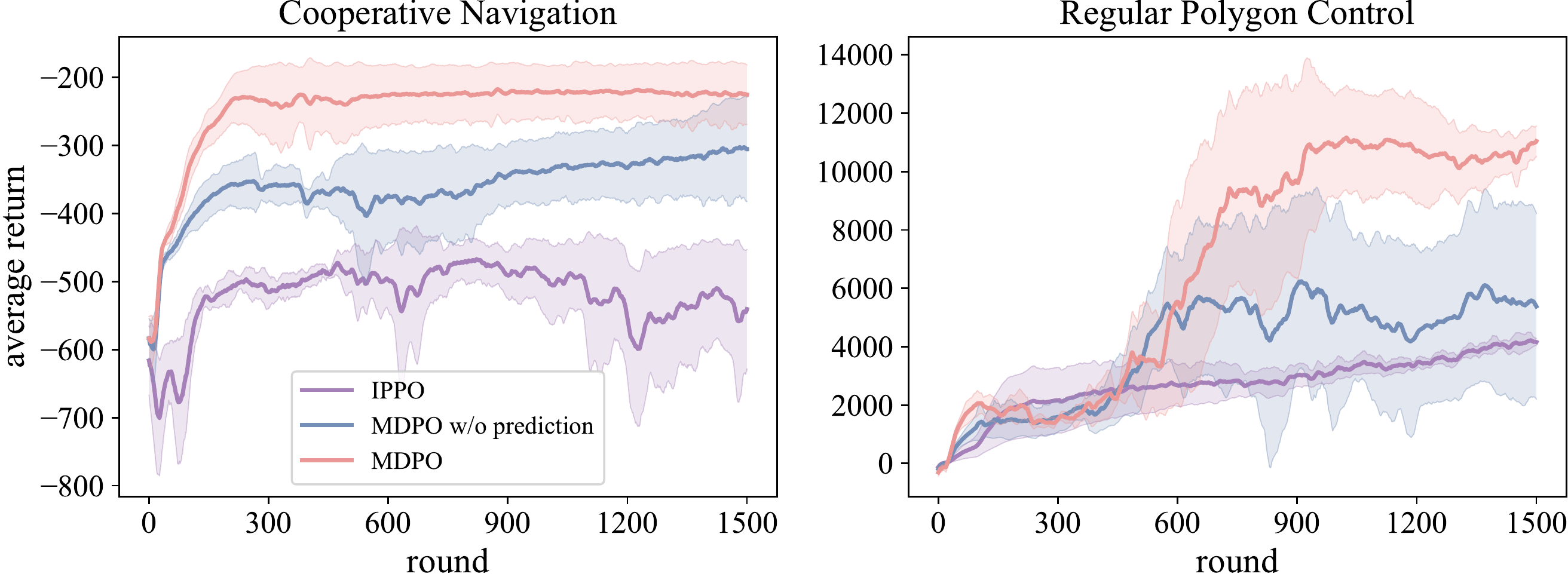}
        \vspace{-0.7cm}
        \caption{Learning curves of MDPO, MDPO w/o prediction, and IPPO in Cooperative Navigation (left) and Regular Polygon Control (right). Each round is 1280 environment steps.}
        \label{mpe}
\end{figure}

Figure \ref{mpe} shows the learning curves of all methods. Generally, MDPO w/o prediction performs better than IPPO, which verifies that the latent variable model can help decentralized policy improvement by making the observation visitation frequency more stationary. And MDPO outperforms MDPO w/o prediction, which verifies latent variable prediction can reduce the gap between the return of interaction and the return predicted by the environment model.

\begin{figure*}[ht]
        \centering
        \includegraphics[width=0.8\textwidth]{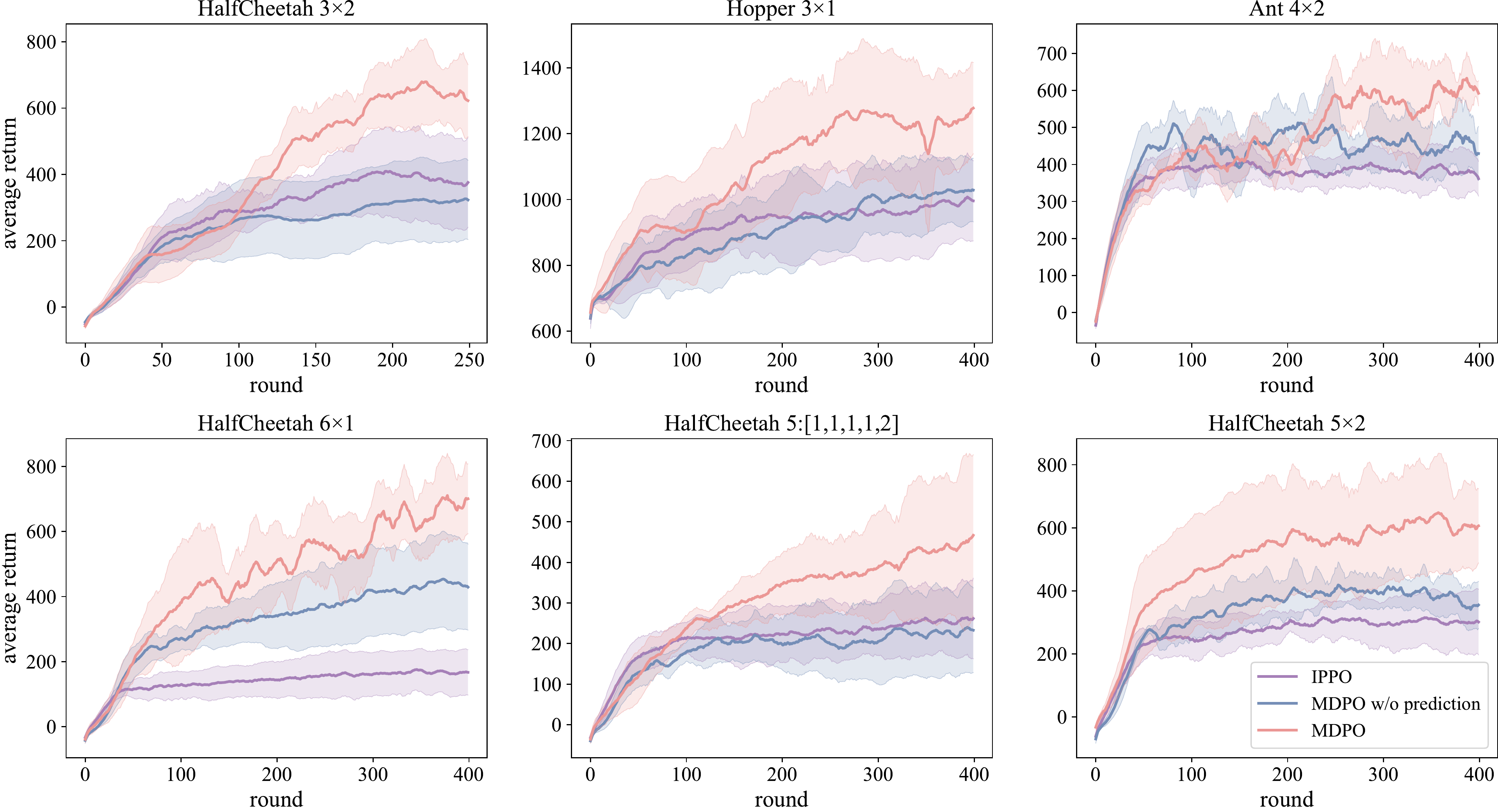}
        \vspace{-0.2cm}
        \caption{Learning curves of MDPO, MDPO w/o prediction, and IPPO in six multi-agent MuJoCo tasks. Each round is 4000 environment steps for 4-agent Ant and 2000 for other tasks.}
        \label{mujoco}
\end{figure*}

It is worth noting that the unobservable information required to fulfill the goal is at completely different levels in the two tasks. Concretely, acknowledging the general direction of others is enough to decide which landmark to approach in Cooperative Navigation. Yet the precise positions of others matter to form a regular polygon in Regular Polygon Control and are hard to learn accurately. Thus, MDPO performs well in Cooperative Navigation since the very early learning stage, while it does not perform well in Regular Polygon Control before 600 rounds. Although a more accurate model is required in Regular Polygon Control, MDPO still converges to better performance. And this indicates a progressive pattern in prediction also works when prediction is hard to be fairly accurate. 

\subsection{Multi-Agent MuJoCo}

Multi-agent MuJoCo is a continuous multi-agent robotic control environment, based on OpenAI's Mujoco Gym environments. In a multi-agent MuJoCo task, each agent controls several joints of the robotic to move forward, where both the observation space and action space are continuous. We choose 3-agent Hopper, 4-agent Ant, and 4 versions of HalfCheetah with different agent numbers or joint allocation for performance comparison. Details of joint allocation are given in Appendix~\ref{sec:hyperparameters}.

As illustrated in Figure \ref{mujoco}, MDPO consistently performs better in these tasks with different difficulties and various agent numbers. Compared with MPE, agents in multi-agent MuJoCo have deeper impacts on each other due to the interaction between adjacent joints. Consequently, the transitions of each agent are closely related to the policies of other agents. Thus, non-stationarity caused by policy updates of other agents is severer in these tasks, resulting in IPPO struggling and converging to low performance. Moreover, note that MDPO w/o prediction performs almost the same as IPPO or even worse in some tasks. The poor performance of MDPO w/o prediction is a consequence of a larger $\epsilon_\psi$ caused by strongly associated agents in these tasks. Thus, latent variable prediction is necessary in these tasks with closely associated agents.

\subsection{Google Research Football}
\label{sec:grf}

In GRF, we choose `simple115v2' as the observation representation which encodes the state with 115 floats and `scoring+checkpoint' as reward which encodes the domain knowledge that scoring is aided by advancing across the pitch. More experimental details are available in Appendix~\ref{app:grf}. 

We compare MDPO, MDPO w/o prediction and ITRPO in two tasks, \textit{Run and Pass} and \textit{3 vs 1 with Keeper}. The experiment is run for about 3M environmental steps and the final average goal rate is reported in Table~\ref{tab:grf}. Although GRF is not indeed a good environment for model learning due to its sparse reward setting, MDPO still improves the average goal rate of ITRPO in both tasks. This verifies the effectiveness of MDPO in more complex environments.

\begin{table}[t]
    \begin{center}
    \caption{Average Goal Rates (\%) in two GRF tasks.}
    \vspace{2mm}
    \label{tab:grf}
    \renewcommand{\arraystretch}{1.3}
    \setlength{\tabcolsep}{4pt}
    \begin{small}
    \begin{tabular}{|c||c|c|c|}
    \hline
          & ITRPO & MDPO w/o pred. & MDPO  \\
    \hline\hline
        \textit{Run and Pass} & $48 \pm 4$ & $53 \pm 4$ & $\boldsymbol{57 \pm 5} $ \\
    \hline
         \textit{3 vs 1 with Keeper} & $27 \pm 4$ & $ 30 \pm 2$ & $\boldsymbol{36 \pm 4}$ \\
    \hline
    \end{tabular}
    \end{small}    
    \end{center}
\end{table}

\section{Conclusion}

In this paper, we propose model-based decentralized policy optimization (MDPO). By introducing a latent variable into the environment model, we theoretically show the model helps to generate experiences with more stationary observation visitation frequency and benefits decentralized policy optimization. Furthermore, We theoretically analyze that the return bound for monotonic policy improvement is controllable by the prediction error of the latent variable function. Consequently, we propose a latent variable prediction method to constrain the prediction error. We examine all the theories and designs via experiments on a set of cooperative multi-agent tasks. Results verify our theoretical results and show MDPO indeed obtains superior performance over model-free decentralized policy optimization. 
\clearpage
\bibliography{reference}
\bibliographystyle{icml2023}

\newpage
\appendix
\onecolumn

\section{Proofs}
\label{sec:proofs}

\subsection{Observation Visitation Frequency Divergence}

In this section, we provide proofs for the upper bound of observation visitation frequency divergence.

\begin{lemma}\label{l.A.1}
Given two pairs of policy and latent variable function, $(\pi_1,\psi_1)$ and $(\pi_2,\psi_2)$. $\forall o\in \mathcal{O}$, it holds that $$\sum_{a,z}{\left| \pi _1\left( a|o \right) \psi _1\left( z|o \right) -\pi _2\left( a|o \right) \psi _2\left( z|o \right) \right|} \le \left| \mathcal{A} \right|\cdot \left\| \pi _1-\pi _2 \right\| +|\mathcal{Z}|\cdot \left\| \psi _1-\psi _2 \right\| ,$$ where $\left\| \pi _1-\pi _2 \right\|\triangleq\underset{a,o}{\max}|\pi_1(a|o)-\pi_2(a|o)|$, $\left\| \psi _1-\psi _2 \right\|\triangleq\underset{z,o}{\max}|\psi_1(z|o)-\psi_2(z|o)|$.
\end{lemma}
\begin{proof}
\begin{align*}
\sum_{a,z}{\left| \pi _1\left( a|o \right) \psi _1\left( z|o \right) -\pi _2\left( a|o \right) \psi _2\left( z|o \right) \right|} \le& \sum_{a,z}{\left| \pi _1\left( a|o \right) \psi _1\left( z|o \right) -\pi _1\left( a|o \right) \psi _2\left( z|o \right) \right|}
\\
&+\sum_{a,z}{\left| \pi _1\left( a|o \right) \psi _2\left( z|o \right) -\pi _2\left( a|o \right) \psi _2\left( z|o \right) \right|}
\\
=&\sum_a{\pi _1\left( a|o \right) \sum_z{\left| \psi _1\left( z|o \right) -\psi _2\left( z|o \right) \right|}}
\\
&+\sum_z{\psi _2\left( z|o \right) \sum_a{\left| \pi _1\left( a|o \right) -\pi _2\left( a|o \right) \right|}}
\\
\le& \left| \mathcal{A} \right|\cdot \left\| \pi _1-\pi _2 \right\| +|\mathcal{Z}|\cdot \left\| \psi _1-\psi _2 \right\|. 
\end{align*}

\end{proof}

\begin{lemma}[Timestep observation visitation frequency recursion]\label{th.a.1}
Given two pairs of policy and latent variable function $(\pi_1,\psi_1)$ and $(\pi_2,\psi_2)$, we define :
\begin{align*}
\Delta \rho _{t}^{\left( \pi _1,\psi _1 \right) ,\left( \pi _2,\psi _2 \right)}\left( o \right) \triangleq \rho _{t}^{\pi _1,\psi _1}\left( o \right) -\rho _{t}^{\pi _2,\psi _2}\left( o \right) ,
\\
\left\| \Delta \rho _{t}^{\left( \pi _1,\psi _1 \right) ,\left( \pi _2,\psi _2 \right)} \right\| \triangleq \underset{o}{\max}\left|\Delta \rho _{t}^{\left( \pi _1,\psi _1 \right) ,\left( \pi _2,\psi _2 \right)}\left( o \right)\right|.
\end{align*}

It holds that 
$$\left| \Delta \rho _{t+1}^{\left( \pi _1,\psi _1 \right) ,\left( \pi _2,\psi _2 \right)}\left( o^\prime \right) \right|\le \left| \mathcal{A} \right|\cdot \left\| \pi _1-\pi _2 \right\| +|\mathcal{Z}|\cdot \left\| \psi _1-\psi _2 \right\| +\left| \mathcal{O} \right|\cdot \left\| \Delta \rho _{t}^{\left( \pi _1,\psi _1 \right) ,\left( \pi _2,\psi _2 \right)} \right\| .$$
\end{lemma}
\begin{proof}
For observation visitation frequency at timestep $t+1$, there is a recurrence relation:
$$
\rho _{t+1}^{\pi ,\psi}\left( o^\prime \right) =\sum_o{\rho _{t}^{\pi ,\psi}\left( o \right) \sum_{a,z}{P\left( o^\prime|a,o,z \right) \pi \left( a|o \right) \psi \left( z|o \right)}}\,\,
$$

Thus, the divergence of observation visitation frequency at timestep $t+1$ can be processed correspondingly:
\small
\begin{align*}
\Delta \rho _{t+1}^{\left( \pi _1,\psi _1 \right) ,\left( \pi _2,\psi _2 \right)}\left( o^\prime \right)=& \rho _{t+1}^{\pi _1,\psi _1}\left( o^\prime \right) -\rho _{t+1}^{\pi _2,\psi _2}\left( o^\prime \right)
\\
=&\sum_o{\left(\rho _{t}^{\pi _1,\psi _1}\left( o \right) \sum_{a,z}{\left( P\left( o\prime|a,o,z \right) \pi _1\left( a|s \right) \psi _1\left( z|o \right) \right)} \right)}
\\
&-\sum_o{\left( \rho _{t}^{\pi _2,\psi _2}\left( o \right) \sum_{a,z}{\left( P\left( o\prime|a,o,z \right) \pi _2\left( a|s \right) \psi _2\left( z|o \right) \right)} \right)}
\\
=&\sum_o{\left( \rho _{t}^{\pi _1,\psi _1}\left( o \right) \sum_{a,z}{\left( P\left( o\prime|a,o,z \right) \left( \pi _1\left( a|o \right) \psi _1\left( z|o \right) -\pi _2\left( a|o \right) \psi _2\left( z|o \right) \right) \right)} \right)}
\\
&+\sum_o{\left( \Delta \rho _{t}^{\left( \pi _1,\psi _1 \right) ,\left( \pi _2,\psi _2 \right)}\left( o \right) \sum_{a,z}{\left( P\left( o\prime|a,o,z \right) \pi _2\left( a|o \right) \psi _2\left( z|o \right) \right)} \right)}   
\end{align*}

Using Lemma \ref{l.A.1}, we can bound the divergence of observation state frequency at timestep $t+1$:

\begin{align*}
\left| \Delta \rho _{t+1}^{\left( \pi _1,\psi _1 \right) ,\left( \pi _2,\psi _2 \right)}\left( o^\prime \right) \right|\le& \sum_o{\left( \rho _{t}^{\pi _1,\psi _1}\left( o \right) \sum_{a,z}{\left( P\left( o^\prime|a,o,z \right) \left| \pi _1\left( a|o \right) \psi _1\left( z|o \right) -\pi _2\left( a|o \right) \psi _2\left( z|o \right) \right| \right)} \right)}
\\
&+\sum_o{\left( \left| \Delta \rho _{t}^{\left( \pi _1,\psi _1 \right) ,\left( \pi _2,\psi _2 \right)}\left( o \right) \right|\sum_{a,z}{\left( P\left( o^\prime|a,o,z \right) \pi _2\left( a|o \right) \psi _2\left( z|o \right) \right)} \right)}
\\
\le& \sum_o{\left( \rho _{t}^{\pi _1,\psi _1}\left( o \right) \left( \left| \mathcal{A} \right|\left\| \pi _1-\pi _2 \right\| +|\mathcal{Z}|\left\| \psi _1-\psi _2 \right\| \right) \right)}
\\&+\left\| \Delta \rho _{t}^{\left( \pi _1,\psi _1 \right) ,\left( \pi _2,\psi _2 \right)} \right\| \sum_{o,a,z}{\left( P\left( o^\prime|a,o,z \right) \pi _2\left( a|o \right) \psi _2\left( z|o \right) \right)}
\\
\le& \left| \mathcal{A} \right|\cdot \left\| \pi _1-\pi _2 \right\| +|\mathcal{Z}|\cdot \left\| \psi _1-\psi _2 \right\| +\left| \mathcal{O} \right|\cdot \left\| \Delta \rho _{t}^{\left( \pi _1,\psi _1 \right) ,\left( \pi _2,\psi _2 \right)} \right\| 
\end{align*}

\end{proof}

\begin{lemma}[discounted observation visitation frequency divergence bound] \label{l.3}
    Given two pairs of policy and latent variable function $(\pi_1,\psi_1)$ and $(\pi_2,\psi_2)$, with the same distribution of initial observation $\rho _{0}\left(o\right)$, it holds that 
$$ \left\| \Delta \rho ^{\left( \pi _1,\psi _1 \right) ,\left( \pi _2,\psi _2 \right)} \right\|_1 \le C\left( \left\| \pi _1-\pi _2 \right\| +\left\| \psi _1-\psi _2 \right\| \right),$$ where $C$ is a certain constant.
\end{lemma}

\begin{proof}
We transform it to the cumulative form of the timestep, and scale it using Lemma \ref{th.a.1}: 

\begin{align*}
\left| \Delta \rho ^{\left( \pi _1,\psi _1 \right) ,\left( \pi _2,\psi _2 \right)}\left( o \right) \right| =&\left| \sum_{t=0}^{\infty}{\gamma ^t\Delta \rho _{t}^{\left( \pi _1,\psi _1 \right) ,\left( \pi _2,\psi _2 \right)}\left( o \right)} \right|
\\
\le& \sum_{t=1}^{T-1}{\gamma ^t\left| \Delta \rho _{t}^{\left( \pi _1,\psi _1 \right) ,\left( \pi _2,\psi _2 \right)}\left( o \right) \right|}+\gamma ^T\sum_{t=T}^{\infty}{\gamma ^{t-T}\left| \Delta \rho _{t}^{\left( \pi _1,\psi _1 \right) ,\left( \pi _2,\psi _2 \right)}\left( o \right) \right|}
\\
\le& \sum_{t=1}^{T-1}{\gamma ^t\left( \left| \mathcal{A} \right|\left\| \pi _1-\pi _2 \right\| +\left| \mathcal{Z} \right|\left\| \psi _1-\psi _2 \right\| \right)}
\\
&+\gamma \left| \mathcal{O} \right|\sum_{t=1}^{T-2}{\gamma ^t\left\| \Delta \rho _{t}^{\left( \pi _1,\psi _1 \right) ,\left( \pi _2,\psi _2 \right)} \right\|}+\frac{2\gamma ^T}{1-\gamma}
\\
\le& \left( \sum_{t=1}^{T-1}{\gamma ^t\sum_{k=0}^{T-1-t}{\left( \gamma \left| \mathcal{O} \right| \right) ^{k}}} \right) \left( \left| \mathcal{A} \right|\left\| \pi _1-\pi _2 \right\| +\left| \mathcal{Z} \right|\left\| \psi _1-\psi _2 \right\| \right) +\frac{2\gamma ^T}{1-\gamma}.
\end{align*}
Thus, we get bound discounted observation visitation frequency divergence:
$$
\left\| \Delta \rho ^{\left( \pi _1,\psi _1 \right) ,\left( \pi _2,\psi _2 \right)} \right\|_\infty \le C_1\left( \left\| \pi _1-\pi _2 \right\| +\left\| \psi _1-\psi _2 \right\| \right),
$$
$$
\left\| \Delta \rho ^{\left( \pi _1,\psi _1 \right) ,\left( \pi _2,\psi _2 \right)} \right\|_1 \le |\mathcal{O}|\cdot \left\| \Delta \rho ^{\left( \pi _1,\psi _1 \right) ,\left( \pi _2,\psi _2 \right)} \right\|_\infty \le C_2\left( \left\| \pi _1-\pi _2 \right\| +\left\| \psi _1-\psi _2 \right\| \right).
$$
\end{proof}

\subsection{Latent Variable Function Divergence}
In this section, we provide proof for divergence comparison between latent variable function of policy rollout and learned latent variable function in the model.

\begin{lemma}[Latent variable function divergence comparison]\label{l.4}
Assume $\psi^n_\omega = (1-\alpha)\psi^{n-1}_\omega + \alpha\psi^n$, and initially $\psi^1_\omega = \psi^1$, where $n$ is the $n$th policy rollout. Then, 
$$
\mathcal{E}_\psi>\mathcal{E}_{\psi_\omega},
$$
where $\mathcal{E}_\psi\triangleq \max\limits_n\left\|\psi^n-\psi^{n-1}\right\| $ and $\mathcal{E}_{\psi_\omega} \triangleq \max\limits_n\left\|\psi^n_\omega-\psi^{n-1}_\omega\right\|$.
\end{lemma}

\begin{proof}
Firstly, we can construct such a recursive inequality:
\begin{align*}
\left\| \psi ^n-\psi _{\omega}^{n-1} \right\| =\left\| \psi ^n-\left( 1-\alpha \right) \psi _{\omega}^{n-2}-\alpha \psi ^{n-1} \right\| \le \left\| \psi ^n-\psi ^{n-1} \right\| +\left( 1-\alpha \right) \left\| \psi ^{n-1}-\psi _{\omega}^{n-2} \right\| .
\end{align*}

Thus, we can expand it recursively:
\begin{align*}
\left\| \psi ^n-\psi _{\omega}^{n-1} \right\| \le& \left\| \psi ^n-\psi ^{n-1} \right\| +\left( 1-\alpha \right) \left\| \psi ^{n-1}-\psi _{\omega}^{n-2} \right\| 
\\
\le& \left\| \psi ^n-\psi ^{n-1} \right\| +\left( 1-\alpha \right) \left\| \psi ^{n-1}-\psi ^{n-2} \right\| +\cdots +\left( 1-\alpha \right) ^{n-1}\left\| \psi ^1-\psi _{\omega}^{0} \right\| 
\\
=&\left\| \psi ^n-\psi ^{n-1} \right\| +\left( 1-\alpha \right) \left\| \psi ^{n-1}-\psi ^{n-2} \right\| +\cdots +\left( 1-\alpha \right) ^{n-1}\left\| \psi ^1-\psi ^0 \right\| 
\\
\le& \mathcal{E}_\psi\left( 1+\left( 1-\alpha \right) +\cdots +\left( 1-\alpha \right) ^{n-1} \right) 
\\
<&\frac{\mathcal{E}_\psi}{\alpha}.
\end{align*}

Using this inequality, we can zoom $\left\|\psi^n_\omega-\psi^{n-1}_\omega\right\|$:
$$
\left\| \psi _{\omega}^{n}-\psi _{\omega}^{n-1} \right\| =\left\| \left( 1-\alpha \right) \psi _{\omega}^{n-1}+\alpha \psi ^n-\psi _{\omega}^{n-1} \right\| =\alpha \left\| \psi ^n-\psi _{\omega}^{n-1} \right\| z<\mathcal{E}_\psi.
$$
Thus, $\mathcal{E}_{\psi_\omega} = \max\limits_n\left\|\psi^n_\omega-\psi^{n-1}_\omega\right\|<\mathcal{E}_\psi.$

\end{proof}

Now we combine Lemma \ref{l.3} and \ref{l.4} to prove Theorem \ref{th.1}.

\textbf{Theorem 3.1} (Latent variable model benefits)\textbf{.} \label{th.1.p}
\textit{Define $\Delta \rho^n(o)\triangleq\rho^{\pi^{n},\psi^{n}}(o)-\rho^{\pi^{n-1},\psi^{n-1}}(o)$. Denote $ \left\| \Delta \rho ^n \right\| \triangleq \underset{o}{\max}|\rho^n(o)-\psi^{n-1}(o)|$, similarly for $\left\| \Delta \pi ^n \right\|$ and $\left\| \Delta \psi ^n \right\|$. It holds that,
\begin{equation*}
    \left\| \Delta \rho ^n \right\| \le C\left( \mathcal{E}_\pi + \mathcal{E}_\psi \right),
\end{equation*}
where $\mathcal{E}_\pi \triangleq \max\limits_n\left\|\Delta\pi^n\right\|$, $\mathcal{E}_\psi \triangleq \max\limits_n\left\|\Delta\psi^n\right\|$ and $C$ is a constant. Assume $\psi^n_\omega = (1-\alpha)\psi^{n-1}_\omega + \alpha\psi^n$ and $\psi^0_\omega = \psi^0$ \footnote{Since $\psi$ is varying and $\psi_\omega$ is continuously updated using the experiences from several recent policy rollouts, we use the form of soft-update for the relation between $\psi$ and $\psi_{\omega}$.}. It holds that $\mathcal{E}_\psi > \mathcal{E}_{\psi_\omega}$ and the bound above is lower when substituting $\psi$ with $\psi_\omega$.}

\begin{proof}
Using Lemma \ref{l.3}, we can scale $ \left\| \Delta \rho ^n \right\|$,
\begin{equation*}
     \left\| \Delta \rho ^n \right\| \le C \left( \left\| \Delta \pi ^n \right\| + \left\| \Delta \psi ^n \right\| \right) \le C\left( \mathcal{E}_\pi + \mathcal{E}_\psi \right)
\end{equation*}
Using Lemma \ref{l.4}, $\mathcal{E}_\psi > \mathcal{E}_{\psi_\omega}$. Thus, the bound above is lower for $\mathcal{E}_{\psi_\omega}$.
\end{proof}

\subsection{Lemmas for Return Bound Analysis}
In this section, we prove several lemmas as preparations for return bound analysis. 
\begin{lemma}[TVD bound of joint distribution]\label{l.c.1}
Consider two joint distributions of $n+1$ variables like this: 

\begin{align*}
P_1\left( x,y_1,y_2,\cdots ,y_n \right) =P_1\left( x \right) \cdot \prod_{i=1}^n{P_1\left( y_i|x \right)}
\\
P_2\left( x,y_1,y_2,\cdots ,y_n \right) =P_2\left( x \right) \cdot \prod_{i=1}^n{P_2\left( y_i|x \right)}
\end{align*}

We can bound the total variation distance of the joint distributions as:
\begin{small}
$$
D_{TV}\left( P_1\left( x,y_1,\cdots ,y_n \right) \|P_2\left( x,y_1,\cdots ,y_n \right) \right) 
\le D_{TV}\left( P_1\left( x \right) \|P_2\left( x \right) \right) +\sum_{i=1}^n{\underset{x}{\max}D_{TV}\left( P_1\left( y_i|x \right) \|P_2\left( y_i|x \right) \right)}.
$$
\end{small}

\end{lemma}
\begin{proof}
We start the proof from a basis case when $n=1$: 

\begin{align*}
D_{TV}\left( P_1 \|P_2\right) =&\frac{1}{2}\sum_{x,y}{\left| P_1\left( x,y \right) -P_2\left( x,y \right) \right|}
\\
\le& \frac{1}{2}\sum_{x,y}{\left| P_1\left( x \right) P_1\left( y|x \right) -P_2\left( x \right) P_1\left( y|x \right) \right|+\left| P_2\left( x \right) P_1\left( y|x \right) -P_2\left( x \right) P_2\left( y|x \right) \right|}
\\
=&\frac{1}{2}\sum_{x,y}{\left| P_1\left( x \right) -P_2\left( x \right) \right|\cdot P_1\left( y|x \right) +P_2\left( x \right) \cdot \left| P_1\left( y|x \right) -P_2\left( y|x \right) \right|}
\\
=&\frac{1}{2}\sum_x{\left| P_1\left( x \right) -P_2\left( x \right) \right|}+\sum_x{P_2\left( x \right) D_{TV}\left( P_1\left( y|x \right) \|P_2\left( y|x \right) \right)}
\\
\le& D_{TV}\left( P_1\left( x \right) \|P_2\left( x \right) \right) +\underset{x}{\max}D_{TV}\left( P_1\left( y|x \right) \|P_2\left( y|x \right) \right) 
\end{align*}

Similarly, we can prove the case of multi-variables:

\begin{align*}
D_{TV}\left( P_1 \|P_2 \right) 
\le D_{TV}\left( P_1\left( x \right) \|P_2\left( x \right) \right) +&\underset{x}{\max}D_{TV}\left( P_1\left( y_1,\cdots ,y_n|x \right) \|P_2\left( y_1,\cdots ,y_n|x \right) \right) 
\\
\le D_{TV}\left( P_1\left( x \right) \|P_2\left( x \right) \right) +&\underset{x}{\max}\Big[ {D_{TV}\left( P_1\left( y_1|x \right) \|P_2\left( y_1|x \right) \right)}
\\
&\ +\underset{x}{\max}D_{TV}\left( P_1\left( y_2,\cdots ,y_n|x \right) \|P_2\left( y_2,\cdots ,y_n|x \right) \right)\Big] 
\\
=D_{TV}\left( P_1\left( x \right) \|P_2\left( x \right) \right) +&\underset{x}{\max}D_{TV}\left( P_1\left( y_1|x \right) \|P_2\left( y_1|x \right) \right) 
\\
&\ +\underset{x}{\max}D_{TV}\left( P_1\left( y_2,\cdots ,y_n|x \right) \|P_2\left( y_2,\cdots ,y_n|x \right) \right) 
\\
\le D_{TV}\left( P_1\left( x \right) \|P_2\left( x \right) \right) +&\sum_{i=1}^n{\underset{x}{\max}D_{TV}\left( P_1\left( y_i|x \right) \|P_2\left( y_i|x \right) \right)}
\end{align*}
\end{proof}
Before proving following lemmas, we clarify the premise our discuss in this section is based on. In a Dec-POMDP, denote the co-occurrence probability of tuple $(o,a,z)$ at timestep $t$ as $P^t(o,a,z)\triangleq P(o_t=o,a_t=a,z_t=z)$. 

Consider two Dec-POMDPs different merely in transition function and reward function, $G_1,G_2$. $P_1$ represents the probability in $G_1$ while $P_2$ for $G_2$. Different policies and latent variable functions, $\left(\pi_1,\psi_1\right)$ and $\left(\pi_2,\psi_2\right)$, are used to rollout respectively in $G_1$ and $G_2$, We denote several bound between them:
\begin{align*}
    &\text{reward bound: } r_{\max} \triangleq \underset{o,a,z}{\max}\max\{R_1(o,a,z),R_2(o,a,z)\}; \\
    &\text{policy bound: } \epsilon _\pi\triangleq \underset{o}{\max}D_{TV}\left( \pi_1 \|\pi_2 \right);\\
    &\text{latent variable function bound: } \epsilon _\psi\triangleq \underset{o}{\max}D_{TV}\left( \psi_1 \|\psi_2 \right); \\ 
    &\text{transition function bound: } \epsilon _m\triangleq \underset{t}{\max}\mathbb{E}_{o,a,z\sim P_{2}^{t-1}}D_{TV}\left( P_1\left( o_t|o,a,z \right) \|P_2\left( o_t|o,a,z \right) \right) .
\end{align*}

Additionally, consider the branched model rollout mentioned in Section 3. Policy, latent variable function, transition function, and reward function vary before and after the model rollout branch. We denote these functions via superscripts 'Pre' for function before branch and 'Post' for functions after branch. Correspondingly, when discussing branched model rollout, we extend the bounds above:
\begin{small}
\begin{align*}
    &\text{reward bound: } r_{\max} \triangleq \underset{o,a,z}{\max}\max\{R_1^{Pre}(o,a,z),R_1^{Post}(o,a,z),R_2^{Pre}(o,a,z),R_2^{Post}(o,a,z)\}; \\
    &\text{policy bound: } \epsilon _\pi^{Pre}\triangleq \underset{o}{\max}D_{TV}\left( \pi_1^{Pre} \|\pi_2^{Pre} \right), \epsilon _\pi^{Post}\triangleq \underset{o}{\max}D_{TV}\left( \pi_1^{Post} \|\pi_2^{Post} \right);\\
    &\text{latent variable function bound: } \epsilon _\psi^{Pre}\triangleq \underset{o}{\max}D_{TV}\left( \psi_1^{Pre} \|\psi_2^{Pre} \right), \epsilon _\psi^{Post}\triangleq \underset{o}{\max}D_{TV}\left( \psi_1^{Post} \|\psi_2^{Post} \right); 
\end{align*}
\begin{align*}
    \text{transition function bound: } \epsilon _m^{Pre}&\triangleq \underset{t}{\max}\mathbb{E}_{o,a,z\sim P_{2}^{t-1,Pre}}D_{TV}\left( P_1^{Pre}\left( o_t|o,a,z \right) \|P_2^{Pre}\left( o_t|o,a,z \right) \right), \\
    \epsilon_m^{Post}&\triangleq \underset{t}{\max}\mathbb{E}_{o,a,z\sim P_{2}^{t-1,Post}}D_{TV}\left( P_1^{Post}\left( o_t|o,a,z \right) \|P_2^{Post}\left( o_t|o,a,z \right) \right) .
\end{align*}    
\end{small}

\begin{lemma}[Observation distributions TVD bound]\label{l.c.2}
The total variation distance of observation distributions at timestep $t$, $P_1(o_t)$ and $P_2(o_t)$, can be bounded as below:
\begin{align*}
D_{TV}\left( P_1\left( o_t \right) \|P_2\left( o_t \right) \right) \le t\left( \epsilon _\pi+\epsilon _\psi+\epsilon _m \right) 
\end{align*}
.
\end{lemma}
\begin{proof}
\begin{align*}
D_{TV}\left( P_1\left( o_t \right) \|P_2\left( o_t \right) \right) =&\frac{1}{2}\sum_{o_t}{\left| P_1\left( o_t \right) -P_2\left( o_t \right) \right|}
\\
=&\frac{1}{2}\sum_{o_t}{\left| \sum_{o,a,z}{P_{1}^{t-1}\left( o,a,z \right) P_1\left( o_t|o,a,z \right) -P_{2}^{t-1}\left( o,a,z \right) P_2\left( o_t|o,a,z \right)} \right|}
\\
\le& \frac{1}{2}\sum_{o_t}\bigg( \sum_{o,a,z}{\left| P_{1}^{t-1}\left( o,a,z \right) -P_{2}^{t-1}\left( o,a,z \right) \right|P_1\left( o_t|o,a,z \right)}
\\
&\qquad\quad+\sum_{o,a,z}{P_{2}^{t-1}\left( o,a,z \right) \left| P_1\left( o_t|o,a,z \right) -P_2\left( o_t|o,a,z \right) \right|} \bigg)
\\
=&\frac{1}{2}\sum_{o,a,z}{\left| P_{1}^{t-1}\left( o,a,z \right) -P_{2}^{t-1}\left( o,a,z \right) \right|\sum_{o_t}{\begin{array}{c}
	P_1\left( o_t|o,a,z \right)\\
\end{array}}}
\\
&+\mathbb{E}_{o,a,z\sim P_{2}^{t-1}}D_{TV}\left( P_1\left( o_t|o,a,z \right) \|P_2\left( o_t|o,a,z \right) \right) 
\\
=&D_{TV}\left( P_{1}^{t-1}\left( o,a,z \right) \|P_{2}^{t-1}\left( o,a,z \right) \right) + \epsilon_m
\end{align*}
According to Lemma \ref{l.c.1},
\begin{align*}
D_{TV}\left( P_{1}^{t-1}\left( o,a,z \right) \|P_{2}^{t-1}\left( o,a,z \right) \right)\le& D_{TV}\left( P_1\left( o_{t-1} \right) \|P_2\left( o_{t-1} \right) \right) 
\\
&+\underset{o}{\max}D_{TV}\left( \pi _1\|\pi _2 \right) +\underset{o}{\max}D_{TV}\left( \psi _1\|\psi _2 \right) 
\\
=& D_{TV}\left( P_1\left( o_{t-1} \right) \|P_2\left( o_{t-1} \right) \right)+\epsilon_\pi+\epsilon_\psi
\end{align*}

Thus,
\begin{align*}
    D_{TV}\left( P_1\left( o_t \right) \|P_2\left( o_t \right) \right) \le& D_{TV}\left( P_1\left( o_{t-1} \right) \|P_2\left( o_{t-1} \right) \right) + \epsilon_\pi + \epsilon_\psi + \epsilon_m
    \\ \le& D_{TV}\left( P_1\left( o_0 \right) \|P_2\left( o_0 \right) \right) +t\left(\epsilon_\pi + \epsilon_\psi + \epsilon_m\right)
    \\ =& t\left(\epsilon_\pi + \epsilon_\psi + \epsilon_m\right)
\end{align*}
\end{proof}

\begin{lemma}[Rollout return bound]\label{l.c.3}
  The gap between rollout returns in $G_1$ with $\left(\pi_1,\psi_1\right)$ and $G_2$ with $\left(\pi_2,\psi_2\right)$ is bounded as:
 \begin{align*}
\left| \eta _1\left( \pi _1,\psi _1 \right) -\eta _2\left( \pi _2,\psi _2 \right) \right| \le \frac{2r_{\max}}{1-\gamma}\left( \frac{\gamma \left( \epsilon _{\pi}+\epsilon _{\psi}+\epsilon _m \right)}{1-\gamma}+\epsilon _{\pi}+\epsilon _{\psi} \right) 
\end{align*}
\end{lemma}
\begin{proof}
\begin{align*}
\left| \eta _1\left( \pi _1,\psi _1 \right) -\eta _2\left( \pi _2,\psi _2 \right) \right|=&\left| \sum_{o,a,z}{R\left( o,a,z \right) \sum_t{\gamma ^t\left( P_{1}^{t}\left( o,a,z \right) -P_{2}^{t}\left( o,a,z \right) \right)}} \right|
\\
\le& 2r_{\max}\sum_t{\gamma ^tD_{TV}\left( P_1^t\left( o,a,z \right) \|P_2^t\left( o,a,z \right) \right)}
\end{align*}
Using Lemma \ref{l.c.1} and \ref{l.c.2}, 
\small
\begin{align*}
D_{TV}\left( P_1^t\left( o,a,z \right) \|P_2^t\left( o,a,z \right) \right)
\le&  D_{TV}\left( P_1\left( o_t \right) \|P_2\left( o_t \right) \right) +\underset{o}{\max}D_{TV}\left( \pi _1\|\pi _2 \right) +\underset{o}{\max}D_{TV}\left( \psi _1\|\psi _2 \right)
\\
\le&  t\left( \epsilon _{\pi}+\epsilon _{\psi}+\epsilon _m \right) +\epsilon _{\pi}+\epsilon _{\psi}
\end{align*}
Thus,
\begin{align*}
\\
\left| \eta _1\left( \pi _1,\psi _1 \right) -\eta _2\left( \pi _2,\psi _2 \right) \right| \le \frac{2r_{\max}}{1-\gamma}\left( \frac{\gamma \left( \epsilon _{\pi}+\epsilon _{\psi}+\epsilon _m \right)}{1-\gamma}+\epsilon _{\pi}+\epsilon _{\psi} \right)     
\end{align*}
\end{proof}

\begin{lemma}[Branched rollout return bound]\label{l.c.4}
In $k$-step branched rollout with $h$-step length before the branch taking into consideration, denote the gap between branched rollout returns in $G_1$ with $\left( \left( \pi _{1}^{Pre},\pi _{1}^{Post} \right) ,\left( \psi _{1}^{Pre},\psi _{1}^{Post} \right)\right)$ and $G_2$ with $\left( \left( \pi _{2}^{Pre},\pi _{2}^{Post} \right) ,\left( \psi _{2}^{Pre},\psi _{2}^{Post} \right) \right)$ as:
$$
\left| \eta _1-\eta _2 \right| \triangleq \left| \eta _1^{branch}\left( \left( \pi _{1}^{Pre},\pi _{1}^{Post} \right) ,\left( \psi _{1}^{Pre},\psi _{1}^{Post} \right) \right) -\eta _2^{branch}\left( \left( \pi _{2}^{Pre},\pi _{2}^{Post} \right) ,\left( \psi _{2}^{Pre},\psi _{2}^{Post} \right) \right) \right|
$$
which is bounded as:
\begin{align*}
\left| \eta _1-\eta _2 \right|\le \frac{2r_{\max}}{1-\gamma}\bigg[ &\left( h+\frac{\gamma ^{h+k+1}}{\left( 1-\gamma \right)} \right) \left( \epsilon _{\pi}^{Pre}+\epsilon _{\psi}^{Pre}+\epsilon _{m}^{Pre} \right) +\left( \epsilon _{\pi}^{Pre}+\epsilon _{\psi}^{Pre} \right) \\
 &\qquad\qquad\qquad\qquad\quad+\gamma ^{h+1}\left( k\left( \epsilon _{\pi}^{Post}+\epsilon _{\psi}^{Post}+\epsilon _{m}^{Post} \right) +\epsilon _{\pi}^{Post}+\epsilon _{\psi}^{Post} \right) \bigg] 
\end{align*}
\end{lemma}

\begin{proof}
According to Lemma \ref{l.c.2},
\begin{align*}
D_{TV}\left( P_1\left( o_t \right) \|P_2\left( o_t \right) \right) \le D_{TV}\left( P_1\left( o_{t-1} \right) \|P_2\left( o_{t-1} \right) \right) +\epsilon _{\pi}+\epsilon _{\psi}+\epsilon _m,
\end{align*}
 which stays in branched rollout case. \\
We can discuss $\delta_t \triangleq D_{TV}\left( P_{1}^{t}\left( o,a,z \right) \|P_{2}^{t}\left( o,a,z \right) \right)$ with different $t$ value:\\
when $t\le h$,
\begin{align*}
\delta_t \le t\left( \epsilon _{\pi}^{Pre}+\epsilon _{\psi}^{Pre}+\epsilon _{\omega}^{Pre} \right) +\epsilon _{\pi}^{Pre}+\epsilon _{\psi}^{Pre}
\end{align*}
when $h<t\le h+k$,
\begin{align*}
 \delta_t \le& h\left( \epsilon _{\pi}^{Pre}+\epsilon _{\psi}^{Pre}+\epsilon _{m}^{Pre} \right) +\left( t-h \right) \left( \epsilon _{\pi}^{Post}+\epsilon _{\psi}^{Post}+\epsilon _{m}^{Post} \right) +\epsilon _{\pi}^{Post}+\epsilon _{\psi}^{Post}   
\end{align*}\\
when $t > h+k$,
\begin{align*}
\delta_t \le&h\left( \epsilon _{\pi}^{Pre}+\epsilon _{\psi}^{Pre}+\epsilon _{m}^{Pre} \right) +k\left( \epsilon _{\pi}^{Post}+\epsilon _{\psi}^{Post}+\epsilon _{m}^{Post} \right)+\epsilon _{\pi}^{Post}+\epsilon _{\psi}^{Post} 
\\
&\qquad\qquad\qquad\qquad\qquad\qquad\qquad\qquad+\left( t-h-k \right) \left( \epsilon _{\pi}^{Pre}+\epsilon _{\psi}^{Pre}+\epsilon _{m}^{Pre} \right) +\epsilon _{\pi}^{Pre}+\epsilon _{\psi}^{Pre}
\\
=&\left( t-k \right) \left( \epsilon _{\pi}^{Pre}+\epsilon _{\psi}^{Pre}+\epsilon _{m}^{Pre} \right) +\epsilon _{\pi}^{Pre}+\epsilon _{\psi}^{Pre}+k\left( \epsilon _{\pi}^{Post}+\epsilon _{\psi}^{Post}+\epsilon _{m}^{Post} \right) +\epsilon _{\pi}^{Post}+\epsilon _{\psi}^{Post}    
\end{align*}

Using the inequalities above, we can write:
\begin{align*}
\delta\triangleq&\sum_t{\gamma ^tD_{TV}\left( P_{1}^{t}\left( o,a,z \right) \|P_{2}^{t}\left( o,a,z \right) \right)} 
\\
\le&\sum_{t=0}^h{\bigg[\gamma ^t\Big( t\left( \epsilon _{\pi}^{Pre}+\epsilon _{\psi}^{Pre}+\epsilon _{m}^{Pre} \right) +\epsilon _{\pi}^{Pre}+\epsilon _{\psi}^{Pre} \Big)\bigg]}
\\
&+\sum_{t=h+1}^{h+k}\bigg[\gamma ^t\Big( h\left( \epsilon _{\pi}^{Pre}+\epsilon _{\psi}^{Pre}+\epsilon _{m}^{Pre} \right) +\epsilon _{\pi}^{Post}+\epsilon _{\psi}^{Post} +\left( t-h \right) \left( \epsilon _{\pi}^{Post}+\epsilon _{\psi}^{Post}+\epsilon _{m}^{Post} \right) \Big)\bigg]
\\
&+\sum_{t=h+k+1}^{\infty}\bigg[\gamma ^t\Big( \left( t-k \right) \left( \epsilon _{\pi}^{Pre}+\epsilon _{\psi}^{Pre}+\epsilon _{m}^{Pre} \right) +\epsilon _{\pi}^{Pre}+\epsilon _{\psi}^{Pre}
\\
&\qquad\qquad\qquad\qquad\qquad\qquad\qquad\qquad\qquad+k\left( \epsilon _{\pi}^{Post}+\epsilon _{\psi}^{Post}+\epsilon _{m}^{Post} \right) +\epsilon _{\pi}^{Post}+\epsilon _{\psi}^{Post} \Big)\bigg]
\\
\le&\sum_{t=0}^h\bigg[\gamma ^t\Big( h\left( \epsilon _{\pi}^{Pre}+\epsilon _{\psi}^{Pre}+\epsilon _{m}^{Pre} \right) +\epsilon _{\pi}^{Pre}+\epsilon _{\psi}^{Pre} \Big)\bigg]
\\
&+\sum_{t=h+1}^{h+k}\bigg[\gamma ^t\Big( h\left( \epsilon _{\pi}^{Pre}+\epsilon _{\psi}^{Pre}+\epsilon _{m}^{Pre} \right) +\epsilon _{\pi}^{Pre}+\epsilon _{\psi}^{Pre}
\\
&\qquad\qquad\qquad\qquad\qquad\qquad\qquad\qquad\qquad+k\left( \epsilon _{\pi}^{Post}+\epsilon _{\psi}^{Post}+\epsilon _{m}^{Post} \right) +\epsilon _{\pi}^{Post}+\epsilon _{\psi}^{Post} \Big)\bigg]
\\
&+\sum_{t=h+k+1}^{\infty}\bigg[\gamma ^t\Big( \left( t-k \right) \left( \epsilon _{\pi}^{Pre}+\epsilon _{\psi}^{Pre}+\epsilon _{m}^{Pre} \right) +\epsilon _{\pi}^{Pre}+\epsilon _{\psi}^{Pre}
\\
&\qquad\qquad\qquad\qquad\qquad\qquad\qquad\qquad\qquad+k\left( \epsilon _{\pi}^{Post}+\epsilon _{\psi}^{Post}+\epsilon _{m}^{Post} \right) +\epsilon _{\pi}^{Post}+\epsilon _{\psi}^{Post} \Big)\bigg]
\\
\le&\left( \sum_{t=0}^{\infty}{\gamma ^t}h+\gamma ^{h+k}\sum_{t=0}^{\infty}{\gamma ^tt} \right) \left( \epsilon _{\pi}^{Pre}+\epsilon _{\psi}^{Pre}+\epsilon _{m}^{Pre} \right) +\left( \sum_{t=0}^{\infty}{\gamma ^t} \right) \left( \epsilon _{\pi}^{Pre}+\epsilon _{\psi}^{Pre} \right) 
\\
&+\left( \sum_{t=h+1}^{\infty}{\gamma ^t} \right) \left( k\left( \epsilon _{\pi}^{Post}+\epsilon _{\psi}^{Post}+\epsilon _{m}^{Post} \right) +\epsilon _{\pi}^{Post}+\epsilon _{\psi}^{Post} \right) 
\\
=&\left( \frac{h}{1-\gamma}+\frac{\gamma ^{h+k+1}}{\left( 1-\gamma \right) ^2} \right) \left( \epsilon _{\pi}^{Pre}+\epsilon _{\psi}^{Pre}+\epsilon _{m}^{Pre} \right) +\frac{1}{1-\gamma}\left( \epsilon _{\pi}^{Pre}+\epsilon _{\psi}^{Pre} \right) 
\\
&+\frac{\gamma ^{h+1}}{1-\gamma}\left( k\left( \epsilon _{\pi}^{Post}+\epsilon _{\psi}^{Post}+\epsilon _{m}^{Post} \right) +\epsilon _{\pi}^{Post}+\epsilon _{\psi}^{Post} \right)
\end{align*}
Thus,
\begin{align*}
\left| \eta _1-\eta _2 \right|\le& 2r_{\max}\delta
\\
\le& \frac{2r_{\max}}{1-\gamma}\bigg[\left( h+\frac{\gamma ^{h+k+1}}{\left( 1-\gamma \right)} \right) \left( \epsilon _{\pi}^{Pre}+\epsilon _{\psi}^{Pre}+\epsilon _{m}^{Pre} \right) +\left( \epsilon _{\pi}^{Pre}+\epsilon _{\psi}^{Pre} \right) \\
 &\qquad\qquad\qquad\qquad\qquad\qquad+\gamma ^{h+1}\left( k\left( \epsilon _{\pi}^{Post}+\epsilon _{\psi}^{Post}+\epsilon _{m}^{Post} \right) +\epsilon _{\pi}^{Post}+\epsilon _{\psi}^{Post} \right) \bigg] 
\end{align*}
\end{proof}

\subsection{Proof of return bound}
In this section, we provide proofs of return bound in different cases.

\textbf{Theorem 3.2} (Rollout return bound for decentralized model) \textbf{.}
\textit{Denote the return gap between $n+1$th policy rollout and model rollout with $n$th learned model as $\left| \eta \left( \pi ,\psi ^{n+1} \right) -\eta ^{\operatorname{model}}\left( \pi ,\psi _{\omega}^{n} \right) \right| \triangleq \Delta \eta $ ,which is bounded as:
\begin{align*}
 \Delta \eta \le \underset{C\left(\epsilon _\theta,\epsilon _{\pi}, \textcolor{red!30!orange}{\epsilon _\omega},\textcolor{blue!70!black}{\epsilon _{\psi}} \right)}{\underbrace{\frac{2r_{\max}}{\left( 1-\gamma \right) ^2}\left( \gamma \epsilon_\theta+2\epsilon_{\pi}+2\epsilon _\omega+ \epsilon _{\psi} \right) }}.
\end{align*}}

\begin{proof}
$$
\left| \eta \left( \pi ,\psi ^{n+1} \right) -\eta ^{model}\left( \pi ,\psi _{\omega}^{n} \right) \right|\leqslant \underset{L_1}{\underbrace{\left| \eta \left( \pi ,\psi ^{n+1} \right) -\eta \left( \pi^n_D,\psi ^n \right) \right|}}+\underset{L_2}{\underbrace{\left| \eta \left( \pi^n_D,\psi ^n \right) -\eta ^{model}\left( \pi ,\psi _{\omega}^{n} \right) \right|}}\underset{}{}
$$
Apply Lemma \ref{l.c.3} to $L_1$ and $L_2$:
\begin{align*}
L_1\le& \frac{2r_{\max}}{1-\gamma}\left( \frac{\gamma \left( \epsilon _{\pi}+\epsilon _{\psi} \right)}{1-\gamma}+\epsilon _{\pi}+\epsilon _{\psi} \right) ,
\\
L_2\le& \frac{2r_{\max}}{1-\gamma}\left( \frac{\gamma \left( \epsilon _{\pi}+\epsilon _{\theta}+\epsilon _{\omega} \right)}{1-\gamma}+\epsilon _{\pi}+\epsilon _{\omega} \right) .
\end{align*}
Thus,
$$
\left| \eta \left( \pi ,\psi ^{n+1} \right) -\eta ^{model}\left( \pi ,\psi _{\omega}^{n} \right) \right|\le \frac{2r_{\max}}{\left( 1-\gamma \right) ^2}\left( 2\epsilon _{\pi}+\epsilon _{\psi}+\epsilon _{\omega}+\gamma \epsilon _{\theta} \right) .
$$
\end{proof}
\textbf{Theorem 3.3} (Branched rollout return bound for decentralized model) \textbf{.}
 \textit{Denote the return gap between $n+1$th policy rollout and branched model rollout with $n$th learned model as $\left| \eta \left( \pi ,\psi ^{n+1} \right) -\eta ^{\operatorname{branch}}\left( \left( \pi^n,\pi \right) ,\left( \psi ^n,\psi _{\omega}^{n} \right) \right) \right| \triangleq \Delta \eta^{\operatorname{branch}} $, which is bounded as:
 }
 \begin{equation*}
\Delta \eta^{\operatorname{branch}} \le C\left(\epsilon _\theta,\epsilon _{\pi}, \textcolor{red!30!orange}{\epsilon _\omega}, \textcolor{blue!70!black}{\epsilon _{\psi}} \right).
\end{equation*}
\begin{proof}
\begin{align*}
\delta \triangleq& \left| \eta \left( \pi ,\psi ^{n+1} \right) -\eta ^{branch}\left( \left( \pi _{D}^{n},\pi \right) ,\left( \psi ^n,\psi _{\omega}^{n} \right) \right) \right|
\\
\le& \underset{L_1}{\underbrace{\left| \eta \left( \pi ,\psi ^{n+1} \right) -\eta ^{branch}\left( \left( \pi _{D}^{n},\pi _{D}^{n} \right) ,\left( \psi ^{n},\psi ^{n} \right) \right) \right|}}\\
&+\underset{L_2}{\underbrace{\left| \eta ^{branch}\left( \left( \pi _{D}^{n},\pi _{D}^{n} \right) ,\left( \psi^{n},\psi ^{n} \right) \right) -\eta ^{branch}\left( \left( \pi _{D}^{n},\pi \right) ,\left( \psi ^n,\psi _{\omega}^{n} \right) \right) \right|}}
\end{align*}
Apply Lemma \ref{l.c.4} to $L_1$ and $L_2$:
\begin{align*}
L_1\le& \frac{2r_{\max}}{1-\gamma}\left[ \left( h+\frac{\gamma ^{h+k+1}}{1-\gamma} \right) \left( \epsilon _{\pi}+\epsilon_{\psi} \right) +\left( \epsilon _{\pi}+\epsilon_{\psi} \right) +\gamma ^{h+1}\left( k\left( \epsilon _{\pi}+\epsilon_{\psi}+\epsilon _{\theta} \right) +\epsilon _{\pi}+\epsilon_{\psi} \right) \right] 
\\
=&\frac{2r_{\max}}{1-\gamma}\left[ \left( h+\frac{\gamma ^{h+k+1}}{1-\gamma}+1+\left( k+1 \right) \gamma ^{h+1} \right) \left( \epsilon _{\pi}+\epsilon_{\psi} \right) +k\gamma ^{h+1}\epsilon _{\theta} \right] ,
\\
L_2\le& \frac{2r_{\max}}{1-\gamma}\left[ \left( h+\frac{\gamma ^{h+k+1}}{1-\gamma} \right) \left( \epsilon _{\omega} \right) +\left( \epsilon _{\omega} \right) +\gamma ^{h+1}\left( k\left( \epsilon _{\pi}+\epsilon _{\omega} \right) +\epsilon _{\pi}+\epsilon _{\omega} \right) \right] 
\\
=&\frac{2r_{\max}}{1-\gamma}\left[ \left( h+\frac{\gamma ^{h+k+1}}{1-\gamma}+1 \right) \epsilon _{\omega}+\left( k+1 \right) \gamma ^{h+1}\left(\epsilon _{\pi}+\epsilon _{\omega}\right) \right] .
\end{align*}
Thus,
\begin{align*}
\delta \le \frac{2r_{\max}}{1-\gamma}\left[ \left( \frac{\gamma ^{h+k+1}}{1-\gamma}+h+1+\left( k+1 \right) \gamma ^{h+1} \right) \left( \epsilon _{\pi}+\epsilon _{\psi}+\epsilon _{\omega} \right) +k\gamma ^{h+1}\left( \epsilon _{\theta}+k\epsilon _{\pi} \right) \right]
\end{align*}
\end{proof}

\textbf{Theorem 3.4} (Rollout return bound for decentralized model with prediction error) \textbf{.}
\textit{Denote the return gap of $n+1$th policy rollout and model rollout with $n$th learned model as $\left| \eta \left( \pi ,\psi ^{n+1} \right) -\eta ^{\operatorname{model}}\left( \pi ,\psi _{\omega}^{n} \right) \right| \triangleq \Delta \eta$, which is bounded as:}
\begin{equation*}
\Delta \eta \le C\left(\epsilon _\theta,\epsilon _{\pi},\textcolor{green!65!blue}{\hat{\epsilon}_\omega} \right).
\end{equation*}
\begin{proof}
$$
\left| \eta \left( \pi ,\psi ^{n+1} \right) -\eta ^{model} \left( \pi ,\psi _{\omega}^{n} \right) \right|\leqslant \underset{L_1}{\underbrace{\left| \eta \left( \pi ,\psi ^{n+1} \right) -\eta \left( \pi^n_D,\psi _{\omega}^{n} \right) \right|}}+\underset{L_2}{\underbrace{\left| \eta\left( \pi^n_D,\psi _{\omega}^{n} \right) -\eta ^{model}\left( \pi ,\psi _{\omega}^{n} \right) \right|}}
$$
Apply Lemma \ref{l.c.3} to $L_1$ and $L_2$:
\begin{align*}
L_1\le& \frac{2r_{\max}}{1-\gamma}\left( \frac{\gamma \left( \epsilon _{\pi}+\hat{\epsilon}_{\omega} \right)}{1-\gamma}+\epsilon _{\pi}+\hat{\epsilon}_{\omega} \right) 
\\
L_2\le& \frac{2r_{\max}}{1-\gamma}\left( \frac{\gamma \left( \epsilon _{\pi} +\epsilon _{\theta}\right)}{1-\gamma}+\epsilon _{\pi} \right) 
\end{align*}
Thus,
$$
\left| \eta \left( \pi ,\psi ^{n+1} \right) -\eta ^{model}\left( \pi ,\psi _{\omega}^{n} \right) \right|\le \frac{2r_{\max}}{\left( 1-\gamma \right) ^2}\left( 2\epsilon _{\pi}+\hat{\epsilon}_{\omega}+\gamma \epsilon _{\theta} \right) .
$$
\end{proof}

\textbf{Theorem 3.5} (Branched rollout return bound for decentralized model with prediction error) \textbf{.}
\textit{The return gap of $n+1$th policy rollout and branched model rollout with $n$th learned model as $\left| \eta \left( \pi ,\psi ^{n+1} \right) -\eta ^{\operatorname{branch}}\left( \left( \pi^n,\pi \right) ,\left( \psi ^n,\psi _{\omega}^{n} \right) \right) \right| \triangleq \Delta \eta^{\operatorname{branch}}$ is bounded as:}
\begin{equation*}
\Delta \eta^{\operatorname{branch}} \le C\left(\epsilon _\theta,\epsilon _{\pi}, \textcolor{red!30!orange}{\epsilon _\omega},\textcolor{green!65!blue}{\hat{\epsilon}_\omega} \right).
\end{equation*}
\begin{proof}
\begin{align*}
\delta \triangleq& \left| \eta \left( \pi ,\psi ^{n+1} \right) -\eta ^{branch}\left( \left( \pi _{D}^{n},\pi \right) ,\left( \psi ^n,\psi _{\omega}^{n} \right) \right) \right|
\\
\le& \underset{L_1}{\underbrace{\left| \eta \left( \pi ,\psi ^{n+1} \right) -\eta ^{branch}\left( \left( \pi _{D}^{n},\pi _{D}^{n} \right) ,\left( \psi _{\omega}^{n},\psi _{\omega}^{n} \right) \right) \right|}}\\
&+\underset{L_2}{\underbrace{\left| \eta ^{branch}\left( \left( \pi _{D}^{n},\pi _{D}^{n} \right) ,\left( \psi _{\omega}^{n},\psi _{\omega}^{n} \right) \right) -\eta ^{branch}\left( \left( \pi _{D}^{n},\pi \right) ,\left( \psi ^n,\psi _{\omega}^{n} \right) \right) \right|}}
\end{align*}
Apply Lemma \ref{l.c.4} to $L_1$ and $L_2$:
\begin{align*}
L_1\le& \frac{2r_{\max}}{1-\gamma}\left[ \left( h+\frac{\gamma ^{h+k+1}}{1-\gamma} \right) \left( \epsilon _{\pi}+\hat{\epsilon}_{\omega} \right) +\left( \epsilon _{\pi}+\hat{\epsilon}_{\omega} \right) +\gamma ^{h+1}\left( k\left( \epsilon _{\pi}+\hat{\epsilon}_{\omega}+\epsilon _{\theta} \right) +\epsilon _{\pi}+\hat{\epsilon}_{\omega} \right) \right] 
\\
=&\frac{2r_{\max}}{1-\gamma}\left[ \left( h+\frac{\gamma ^{h+k+1}}{1-\gamma}+1+\left( k+1 \right) \gamma ^{h+1} \right) \left( \epsilon _{\pi}+\hat{\epsilon}_{\omega} \right) +k\gamma ^{h+1}\epsilon _{\theta} \right] ,
\\
L_2\le& \frac{2r_{\max}}{1-\gamma}\left[ \left( h+\frac{\gamma ^{h+k+1}}{1-\gamma} \right) \left( \epsilon _{\omega} \right) +\left( \epsilon _{\omega} \right) +\gamma ^{h+1}\left( k\left( \epsilon _{\pi} \right) +\epsilon _{\pi} \right) \right] 
\\
=&\frac{2r_{\max}}{1-\gamma}\left[ \left( h+\frac{\gamma ^{h+k+1}}{1-\gamma}+1 \right) \epsilon _{\omega}+\left( k+1 \right) \gamma ^{h+1}\epsilon _{\pi} \right] .
\end{align*}
Thus,
\begin{align*}
\delta \le \frac{2r_{\max}}{1-\gamma}\left[ \left( \frac{\gamma ^{h+k+1}}{1-\gamma}+h+1 \right) \left( \epsilon _{\pi}+\hat{\epsilon}_{\omega}+\epsilon _{\omega} \right) +\left( k+1 \right) \gamma ^{h+1}\left( 2\epsilon _{\pi}+\hat{\epsilon}_{\omega} \right) +k\gamma ^{h+1}\epsilon _{\theta} \right] 
\end{align*}
\end{proof}

\section{Verification on Learned Latent Variable}\label{sec:verify}

To examine how related the learned latent variable and the inaccessible information are, we designed a simple tabular case, where policies, transition matrix, and reward matrix are preset. There are 3 states and 3 agents with 2 actions for each and the space of latent variable is set to be 4. For agent $0$, we collect experiences $\langle s,a_0,s^\prime,r \rangle$ and train a latent variable model end-to-end. For visualization, we design $\psi_{\omega}$ as an explicit network to fetch learned $z_0$ and preset the forward pass for the state in $P_\theta$ to avoid the correspondence being conditioned on the state. Then, we sample $z_0$ from learned latent variable function network $\psi_\omega$ for each experience in the buffer, and then calculate the conditional probabilities, $P(z_0|\boldsymbol{a}_{-0})$ and $P(\boldsymbol{a}_{-0}|z_0)$. As shown in Figure \ref{fig:verify}, there is a one-to-one correspondence between the latent variable and other agents' joint action. This demonstrates the latent variable model can implicitly capture inaccessible information relevant to transition and reward via end-to-end learning.

\begin{figure*}[t]
        \centering
        \includegraphics[width=0.65\textwidth]{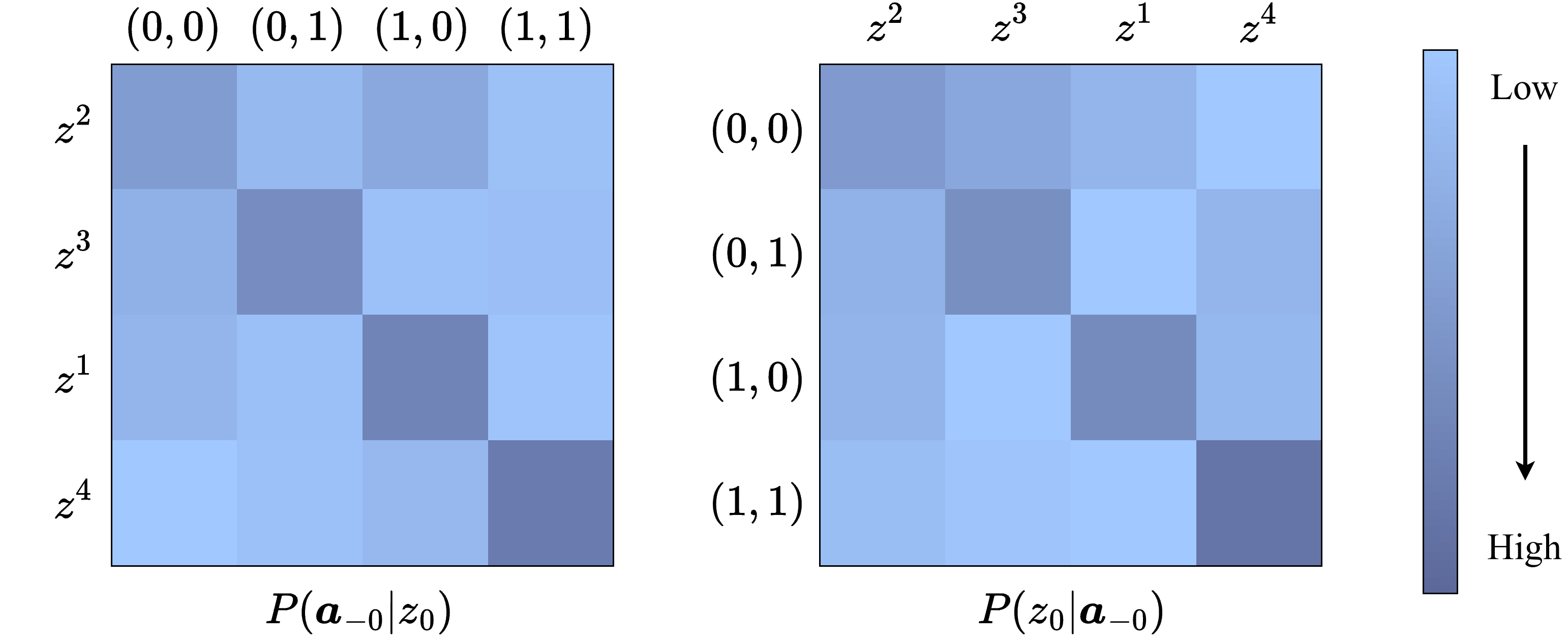}
        \vspace{-2mm}
        \caption{Conditional probability matrices: $P(\boldsymbol{a}_{-0}|z_0)$ (left) and $P(z_0|\boldsymbol{a}_{-0})$ (right). Deeper color means higher probability. The order of $z_0$ has been adjusted for clarification.}
        \label{fig:verify}
\end{figure*}

\section{Experiment Details} \label{sec:hyperparameters}
\subsection{Environment Setting}
In this section, we introduce the environment settings we used in the experiments.

\textbf{Stochastic Game.} In our stochastic game, there are 30 observations, 3 agents and 5 actions for each agent, and episode length is limited to 40 steps. We generate a transition matrix $T$ and a reward matrix $R$ in advance as transition function and reward function. Concretely, $T$ is a matrix in shape of $[30,5,5,5,30]$ and $R$ is a matrix in shape of $[30,5,5,5,1]$. At each timestep $t$, given observation $o_t$ and agent joint actions $\left(a^0_t,a^1_t,a^2_t\right)$, the transition is:
\begin{align*}
    o_{t+1}\sim T[o_t,a^0_t,a^1_t,a^2_t],\ r_t=R[o_t,a^0_t,a^1_t,a^2_t].
\end{align*}

\textbf{MPE.} In our MPE tasks, agents observe their own positions, velocity, and others' relative positions. And actions of agents control their accelerations in every direction which is continuous in our experiments. In MPE tasks, the episode length is limited to 40 steps.
\begin{itemize}[leftmargin=2em]
    \item[] \textbf{4-Agent Cooperative Navigation.} In 4-agent Cooperative Navigation, as shown in Figure~\ref{mpe_task} (left), 4 agents learn to cooperate to reach 4 landmarks respectively. Concretely, we denote the radius of agent $i$ as $d^i$, position of agent $i$ as $(x^i_a,y^i_a)$ and position of landmark $i$ as $(x^i_l,y^i_l)$, and the reward is: $$r=-\sum_{i=0}^{3}{\left(\underset{j}{\min}\sqrt{(x^i_l-x^j_a)^2+(y^i_l-y^j_a)^2}\right)}-p_c,$$ where $p_c$ is a collision penalty:
    $$
    p_c = \sum_{0\le i,j\le 3}\left[\mathbb{I}_{\{x|x<d_i+d_j\}}\left(\sqrt{(x^i_a-x^j_a)^2+(y^i_a-y^j_a)^2}\right)\right].
    $$ 
    Thus, the reward upper bound at each step is $-4$.
    \item[] \textbf{5-Agent Regular Polygon Control.} In 5-agent Regular Polygon Control, as shown in Figure~\ref{mpe_task} (right), 4 agents learn to cooperate with another agent, which is controlled by a fixed policy, aiming to form a regular pentagon. The fixed policy is that the acceleration of the agent is always in the direction of the relative position between the center of the other 4 agents and itself. And the reward is given according to the area $S$ of current pentagon scaled by its perimeter $C$ , which formally is:
    $$
    S_{scaled}=\left\{
    \begin{aligned}
    & S\cdot(\frac{10}{C})^2, & \text{agents form a convex pentagon} \\
    & 0, & \text{otherwise,} 
    \end{aligned}\right.
    $$ 
    and represents the area of its similar pentagon with a perimeter of 10. So when the pentagon is a regular pentagon, $S_{scaled}$ comes to its maximum, $5\cot \frac{\pi}{5}$. Additionally, two penalty items are given. Bound penalty, $p_b$, is used to restrict agents to stay in bounds :
    $$
    bound(x) = \left\{
    \begin{aligned}
    & 0, & |x|<0.9 \\
    & 10*(|x|-0.9), & |x|<1.0\\
    & e^{2|x|-2}, & \text{otherwise}\\
    \end{aligned}\right.,
    $$
    $$
    p_b=\sum_{i=0}^{3}\left[bound(x^i)+bound(y^i)\right],
    $$ where $(x^i,y^i)$ is the position of agent $i$. Collision penalty, $p_c$, is as same as that in 4-agent Cooperative Navigation. Finally, we design the reward as:
    $$
    r=\min\left\{\max\left\{S_{scaled},\frac{1}{5\cot \frac{\pi}{5}-S_{scaled}}\right\},1000\right\} - 4p_c -p_b,
    $$ where $\max$ operator helps to distinguish when pentagon is relatively large and $\min$ operator handles the situation being divided by zero. This task is more difficult than Cooperative Navigation.
\end{itemize}
\begin{figure*}[t]
        \centering
        \includegraphics[width=.6\textwidth]{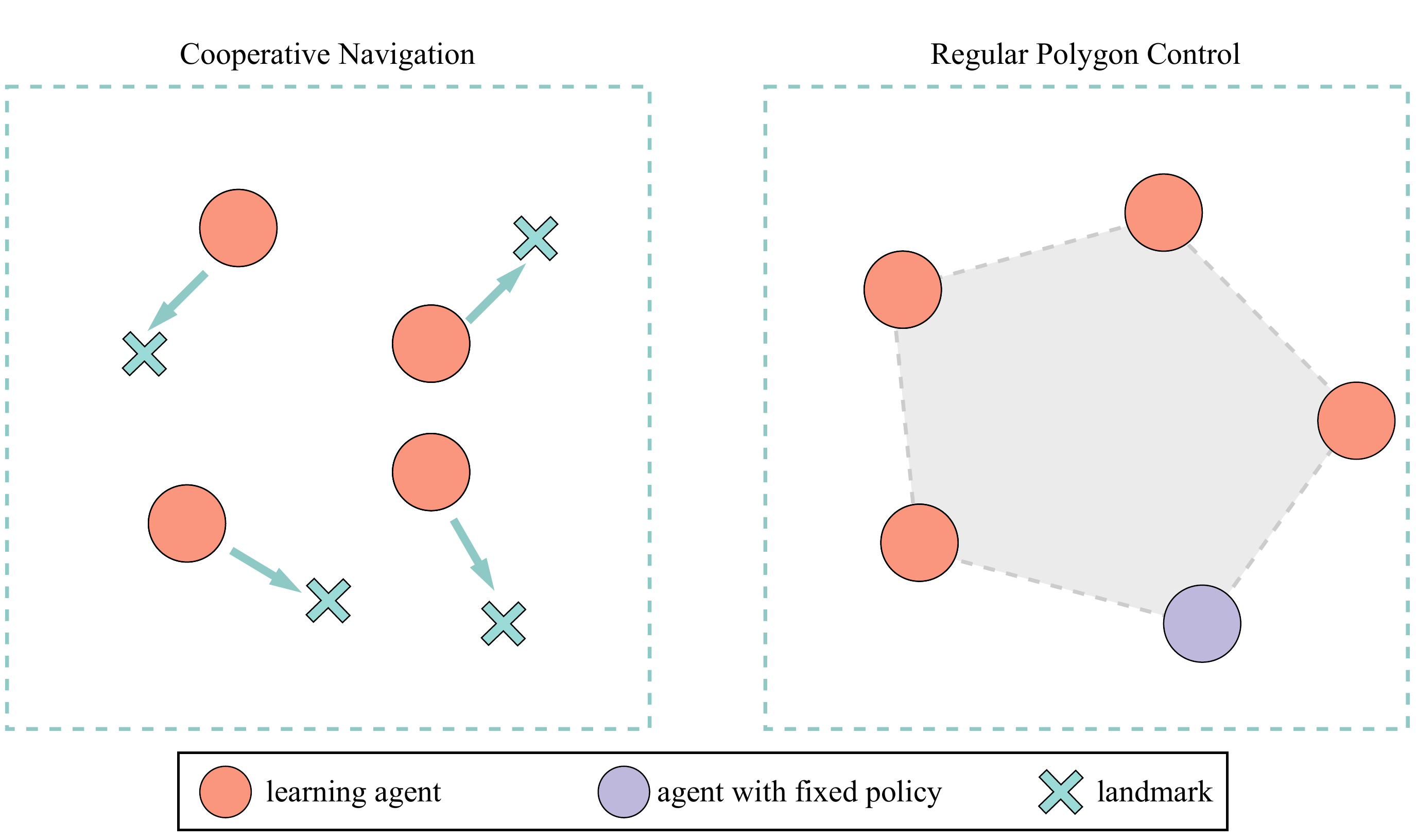}
        \caption{Illustration of MPE tasks: Cooperative Navigation (left) and Regular Polygon Control (right). }
        \label{mpe_task}
\end{figure*}
\textbf{Multi-Agent MuJoCo.} In our multi-agent MuJoCo experiments, the state in MuJoCo environment, which describes the position, velocity, angular velocity of each joint, etc, is used as the observation distributed to each agent. Specifically, in Ant task, we only use dimensions from 0 to 26 of the state. We limit the episode length of Halfcheetah to 250 steps, and 500 steps for Ant and Hopper. We provide the joint allocation of each task in Table~\ref{tab:joint allocation}.
\begin{table}[t]
    \begin{center}
    \caption{Joint allocation in multi-agent MuJoCo tasks. The relation column indicates how agents control the joints of robotics. }
    \label{tab:joint allocation}
    \vspace{2mm}
    \begin{small}
    \begin{tabular}{|c|c|c|c|}
    \hline
         Task & MuJoCo action & Multi-agent MuJoCo actions & Relation  \\
    \hline
         Hopper 3$\times$1& $(a_0,a_1,a_2)$ & $[(a^0_0),(a^1_0),(a^2_0)]$ & $a_i=a^i_0$ \\
    \hline
         Ant 4$\times$2& $(a_0,\cdots,a_7)$ & $[(a^0_0,a^0_1),\cdots,(a^3_0,a^3_1)]$ & $a_{2i+k}=a^i_k$\\
    \hline
         HalfCheetah 3$\times$2& $(a_0,\cdots,a_5)$ & $[(a^0_0,a^0_1),\cdots,(a^2_0,a^2_1)]$ & $a_{2i+k}=a^i_k$\\
    \hline
         HalfCheetah 6$\times$1& $(a_0,\cdots,a_5)$ & $[(a^0_0),\cdots,(a^5_0)]$ & $a_i=a^i_0$\\
    \hline
         HalfCheetah 5:[1,1,1,1,2]& $(a_0,\cdots,a_5)$ & $[(a^0_0),\cdots,(a^3_0),(a^4_0,a^4_1)]$ & $a_{i+k}=a^i_k$\\
    \hline
         HalfCheetah 5$\times$2 & $(a_0,\cdots,a_5)$ & $[(a^0_0,a^0_1),\cdots,(a^4_0,a^4_1)]$ & $ a_i=a^i_0$, $a_5=\frac{\sum_ia^i_1}{5}$\\
    \hline
    \end{tabular}
    \end{small}    
    \end{center}
\end{table}

\begin{table}[t]
    \centering
    \caption{Structure of the neural networks we used in experiments.}
    \label{tab:structure}
    \vspace{2mm}
    \begin{tabular}{|c|c|c|c|c|c|c|}
    \hline
        \multirow{2}{*}{Network} & \multicolumn{2}{|c|}{Stochastic Game} & \multicolumn{2}{|c|}{MPE} & \multicolumn{2}{|c|}{Multi-agent MuJoCo}\\ 
    \cline{2-7}
         & hidden & activation& hidden & activation& hidden & activation\\
    \hline
    PPO Actor & (128,128)& tanh & (128,128)& tanh & (128,128)& tanh\\
    \hline
    PPO Critic & (128,128)& tanh & (128,128)& tanh & (128,128)& tanh \\
    \hline
    $\psi_\omega$ & (64,64) & ReLU & (64,64) & ReLU & (32,32) & ReLU\\
    \hline 
    $f_\zeta$ & (128) & ReLU & (64,64) & ReLU & (256) & ReLU\\
    \hline
    $P_\theta$ & (128,64) & ReLU & (64) & ReLU & (128,64) & ReLU \\
    \hline
    $R_\phi$ & (128,64) & ReLU & (64,64) & ReLU & (128,64) & ReLU \\
    \hline
    \end{tabular}
\end{table}

\begin{table}[t!]
    \centering
    \caption{Hyperparameters}
    \label{tab:hyperparameters}
    \vspace{2mm}
    \begin{tabular}{|c|c|c|c|}
    \hline
        {} & {Stochastic Game} & {MPE} & {Multi-agent MuJoCo}\\
    \hline
        latent variable dimension & 3 & 4 & 6\\
    \hline
        $\lambda$ & \multicolumn{3}{|c|}{0.94}\\
    \hline
        $\gamma$ & \multicolumn{3}{|c|}{0.98}\\
    \hline
        $\epsilon$ & \multicolumn{3}{|c|}{0.2}\\
    \hline
        $\epsilon_{value}$ & \multicolumn{3}{|c|}{10}\\
    \hline
        $c_{entropy}$ & \multicolumn{3}{|c|}{3e-4}\\
    \hline
        $c_{value}$ & \multicolumn{3}{|c|}{0.5}\\
    \hline
        max gradient norm &  \multicolumn{3}{|c|}{0.6}\\
    \hline
        PPO batch size & \multicolumn{3}{|c|}{32}\\
    \hline
        actor learning rate &  \multicolumn{3}{|c|}{3e-4}\\
    \hline
        critic learning rate &  \multicolumn{3}{|c|}{1e-3}\\
    \hline
        $k$ &  \multicolumn{2}{|c|}{4}  & 2\\
    \hline
        $h$ &   \multicolumn{2}{|c|}{8} & 6\\
    \hline
        $l$ &  \multicolumn{3}{|c|}{8}\\
    \hline
        $c_{o\prime}$ & 10 & 1,100\footnotemark[4] & 5\\
    \hline
        latent variable model batch size & \multicolumn{3}{|c|}{64}\\
    \hline
        prediction batch size & \multicolumn{2}{|c|}{32} & 128\\
    \hline
        $\psi_\omega$ learning rate & \multicolumn{3}{|c|}{1e-5}\\
    \hline 
        $f_\zeta$ learning rate & 1e-4 & 3e-5 & 5e-5 \\
    \hline
        $P_\theta$ learning rate & 1e-4 & \multicolumn{2}{|c|}{3e-5}\\
    \hline
        $R_\phi$ learning rate & 1e-4 & \multicolumn{2}{|c|}{3e-5}\\   
    \hline
    \end{tabular}
\end{table}

\subsection{Implementation \& Hyperparameters}
In this section, we provide details for implementation and hyperparameters.

For the experiment environment, we adopt MPE (MIT license) and MuJoCo Gym (MIT license). For PPO, we follow the version in OpenAI's Spinning Up (MIT license).

All neural networks used in our implementation are in the form of Multi-Layer Perception (MLP). Particularly, the transition function and reward function are respectively learned using an ensemble formed by 3 individual versions of the last layer. The hidden size and activation function used in the networks are provided in Table~\ref{tab:structure}. And the parameters used in training are provided in Table~\ref{tab:hyperparameters}.

\footnotetext[4]{$c_{o\prime}=1$ in 4-Agent Cooperative Navigation task and $c_{o\prime}=100$ in 5-Agent Regular Polygon Control}

In the implementation of latent variable function, both deterministic and stochastic latent variable satisfy our analysis, and we choose between them according to environment properties and experimental performance. In MPE and stochastic game except for Appendix B, we use deterministic latent variable with L2 regularization. In stochastic game in Appendix B, we use Category distribution. And in multi-agent MuJoCo environment, we use Gaussian distribution. As for the implementation of transition function and reward function, we use Category distribution for transition function in stochastic game and deterministic output for others.

The experiments are carried out on Intel i9-10900K CPU and NVIDIA GTX 3080Ti GPU. The training of stochastic game task costs 6 hours, while it takes 14 hours for each MPE task, and 25 hours for each multi-agent MuJoCo task. 


\begin{table}[t]
    \centering
    \caption{Structure of neural networks in GRF experiments.}
    \vspace{2mm}
    \label{tab:grf structure}
    \begin{tabular}{|c|c|c|}
    \hline
        \multirow{2}{*}{Network} & \multicolumn{2}{|c|}{Google Research Football} \\ 
    \cline{2-3}
         & hidden & activation\\
    \hline
    TRPO Actor & (128,128)& tanh \\
    \hline
    TRPO Critic & (128,128)& tanh \\
    \hline
    $\psi_\omega$ & (128,128) & ReLU \\
    \hline 
    $f_\zeta$ & (256) & ReLU \\
    \hline
    $P_\theta$ & (256,128) & ReLU \\
    \hline
    $R_\phi$ & (128,128) & ReLU \\
    \hline
    \end{tabular}
\end{table}

\begin{table}[t]
    \centering
    \caption{Hyperparameters in GRF experiments}
    \label{tab:grf hyperparameters}
    \vspace{2mm}
    \begin{tabular}{|c|c|c|c|}
    \hline
        \multicolumn{2}{|c|}{TRPO hyperparameters} & \multicolumn{2}{|c|}{MDPO hyperparameters}\\
    \hline
        $\lambda$ & 1.0 & latent variable dimension & 6\\
    \hline
        $\gamma$ & 0.99 & $k$ & 4 \\
    \hline
        KL limitation & 0.06 & $h$ & 16 \\
    \hline
        damping coefficient & 0.2 & $l$ & 8 \\
    \hline
        conjugate gradient iteration & 8 & $c_{o\prime}$ & 1 \\
    \hline
        backtrack iteration & 8 & latent variable model batch size & 128 \\
    \hline
        backtrack coefficient & 0.8 & prediction batch size & 256\\
    \hline
        max gradient norm & 10 & $\psi_\omega$ learning rate & 1e-5\\
    \hline
        TRPO batch size & 48 & $f_\zeta$ learning rate & 5e-5\\
    \hline
        $\epsilon_{value}$ & 10 & $P_\theta$ learning rate & 3e-5\\
    \hline
        critic learning rate &  5e-5 & $R_\phi$ learning rate & 1e-5 \\
    \hline
    \end{tabular}
\end{table}

\section{Google Research Football}
\label{app:grf}



In our implementation, we use Category distribution for one-hot dimensions in observation and Gaussion distribution for others to model the transition. We use MLP for reward function and Gaussion distribution for latent variable function. The network structures are listed in Table~\ref{tab:grf structure} and the hyperparameters are listed in Table~\ref{tab:grf hyperparameters}.

\section{Single-Agent Non-Stationary Environment}
\label{app:more}

 \begin{figure*}[t]
        \centering
        \includegraphics[width=1\textwidth]{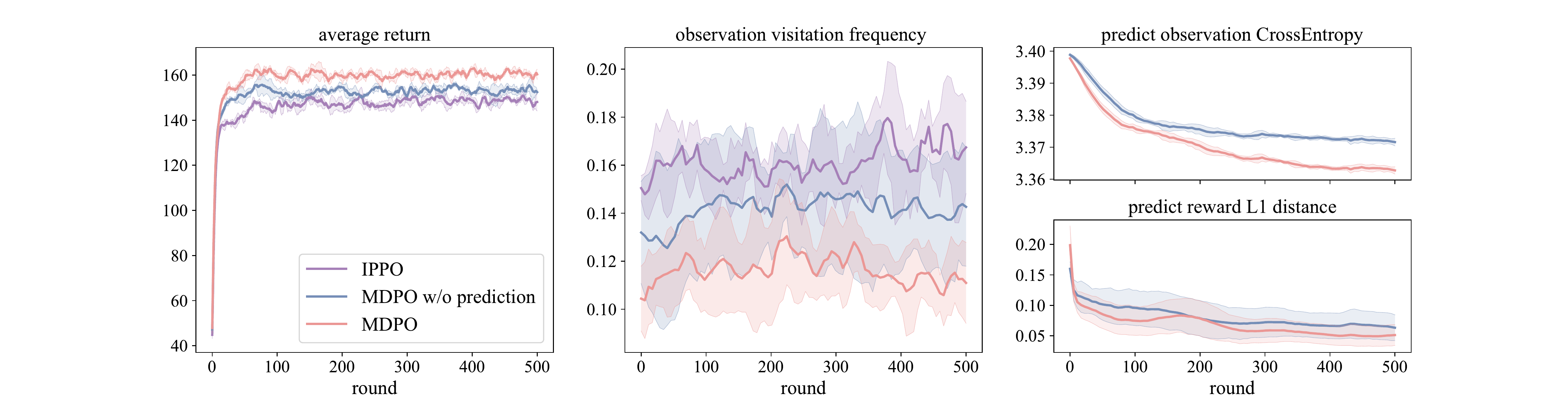}
        \vspace{-6mm}
        \caption{Learning curves of MDPO compared with MDPO w/o prediction and IPPO on the single-agent non-stationary stochastic game: average return (left), observation visitation frequency divergence (mid), and model prediction errors (right). Each round is 1600 environment steps.}
        \label{matrix_single}
\end{figure*}

Since MDPO helps to handle non-stationarity in decentralized MARL from the perspective of an individual agent, it will be natural and easy to also apply MDPO to single-agent RL in non-stationary environments. In this section, we investigate how MDPO performs in such a non-stationary single-agent environment.

We adopt the cooperative stochastic game into a single-agent non-stationary version. Concretely, we fix the policies of two agents and leave only one agent to update its policy. And we generate 5 noise matrices ($N_0,\cdots,N_4$) randomly, which are in the same shape as the transition matrix ($T$) and will influence the transition probability in a rotating manner. Formally, in $n$th policy rollout, the transition matrix is $T + N_{n\bmod5}$, and we guarantee such a transition matrix is reasonable when generating noise matrices.

We compare the performance of MDPO, MDPO w/o prediction, and IPPO on the single-agent non-stationary stochastic game, and the learning curves are shown in Figure~\ref{matrix_single}. As illustrated in Figure~\ref{matrix_single} (left), MDPO still performs the best in the single-agent non-stationary environment. As shown in  Figure~\ref{matrix_single} (right), latent variable prediction helps to predict the non-stationary transition, and as no noise is applied to the reward matrix, there is merely a slight difference in reward prediction. Since the latent variable in this environment (noise matrices) is in a regular rotation, the prediction function is easier to learn than in decentralized MARL settings. However, unlike in decentralized MARL, non-stationarity in this setting will not fade away in pace with policy convergence, thus MDPO w/o prediction may keep oscillating and generate experiences with larger observation visitation frequency divergence than MDPO, which is shown in Figure~\ref{matrix_single} (mid).
 
Generally, MDPO also works in single-agent non-stationary environments, especially when there is a regular pattern of non-stationarity. More thorough studies are left as future work.

\section{Additional Related Work}
\label{app:related}
By utilizing an environment model, model-based RL has shown many advantages, such as sample efficiency \citep{wang2019benchmarking} and exploration \citep{pathak2017curiosity}. Many paradigms have been proposed on how to exploit the environment model. Model-based planning methods, such as model predictive control, select actions through model rollouts. Dyna-style methods \citep{sutton1990integrated, feinberg2018model,janner2019trust} use both data collected in the real environment and data generated by the learned model to update the policy. Recent studies have extended model-based methods to multi-agent settings for sample efficiency in zero-sum game \citep{zhang2020model}, in stochastic game \citep{AORPO} and in networked system \citep{du2022fully}, centralized training \citep{willemsen2021mambpo}, opponent modeling \citep{yu2021model}, and communication \citep{kim2020communication}. However, none of them strictly tackle fully decentralized setting of our paper. 

Specially, the DMPO (decentralized model-based policy optimization) algorithm in prior work \citep{du2022fully} is designed for a networked system, where agents are able to communicate along the edges with their neighbors. The naming of \textit{DMPO} and \textit{MDPO} may lead to misunderstanding of similar settings, but the two algorithms are actually concerned with different settings. In fully decentralized setting of our paper, no information sharing is allowed between agents. And when the number of neighbors is set zero in DMPO, it will degenerate into the version of MDPO w/o prediction in our algorithm.

\end{document}